%% file: ExpST.tex
\documentclass[14pt]{extarticle} % For LaTeX2e
\usepackage{arxiv,times}
\usepackage{wrapfig}
\input{math_commands.tex}

\renewcommand{\eqref}[1]{(\ref{#1})}
\usepackage{hyperref}
\usepackage{url}
\usepackage{natbib}
\usepackage{graphicx}

\usepackage{caption}
\usepackage[utf8]{inputenc} % allow utf-8 input
\usepackage[T1]{fontenc}    % use 8-bit T1 fonts
\usepackage{hyperref}       % hyperlinks
\usepackage{url}            % simple URL typesetting
\usepackage{booktabs}       % professional-quality tables
\usepackage{amsfonts}       % blackboard math symbols
\usepackage{amsmath}
\usepackage{nicefrac}       % compact symbols for 1/2, etc.
\usepackage{microtype}      % microtypography
\usepackage{xcolor}         % colors
\usepackage{bbm}
\usepackage{amsthm}
\usepackage[linesnumbered,ruled,vlined]{algorithm2e}
\SetKwInput{KwInput}{Input}                % Set the Input
\SetKwInput{KwOutput}{Output}

\newtheorem{theorem}{Theorem}[section]
\newtheorem{proposition}[theorem]{Proposition}
\newtheorem{definition}[theorem]{Definition}
\newtheorem{corollary}[theorem]{Corollary}

\newtheorem{remark}[theorem]{Remark}

\newtheorem{lemma}[theorem]{Lemma}

\definecolor{akcolor}{rgb}{0.65, 0.15, 0.6}

\definecolor{darkblue}{rgb}{0.1,0.1,0.6}

\def\cD{\mathcal{D}}

\def\cP{\mathcal{P}}
\def\cU{\mathcal{U}}

\def\bbR{\mathbb{R}}
\def\bbS{\mathbb{S}}
\def\weakstarto{\stackrel{*}{\rightharpoonup}}
\def\Unif{\mathrm{Unif}}
\def\TV{\mathrm{TV}}
\def\supp{\mathrm{supp}}

\title{Expected Sliced Transport Plans}

\author{%
  Xinran Liu$^{1}$, Rocío Martín Díaz$^{2}$, Yikun Bai$^{1}$, Ashkan Shahbazi$^{1}$, \\ \textbf{Matthew Thorpe}$^{3}$, \textbf{Akram Aldroubi}$^{4}$, \textbf{Soheil Kolouri}$^{1}$  \\ 
  \\
  $^{1}$Department of Computer Science, Vanderbilt University, Nashville, TN, 3723\\
  $^{2}$Department of Mathematics, Tufts University, Medford, MA 02155\\
  $^{3}$Department of Statistics, University of Warwick, Coventry, CV4 7AL, UK\\
  $^{4}$Department of Mathematics, Vanderbilt University, Nashville, TN, 3723\\
}

\begin{document}

\date{}
\maketitle

\begin{abstract}

The optimal transport (OT) problem has gained significant traction in modern machine learning for its ability to: (1) provide versatile metrics, such as Wasserstein distances and their variants, and (2) determine optimal couplings between probability measures. To reduce the computational complexity of OT solvers, methods like entropic regularization and sliced optimal transport have been proposed. The sliced OT framework improves efficiency by comparing one-dimensional projections (slices) of high-dimensional distributions. However, despite their computational efficiency, sliced-Wasserstein approaches lack a transportation plan between the input measures, limiting their use in scenarios requiring explicit coupling. In this paper, we address two key questions: Can a transportation plan be constructed between two probability measures using the sliced transport framework? If so, can this plan be used to define a metric between the measures? We propose a "lifting" operation to extend one-dimensional optimal transport plans back to the original space of the measures. By computing the expectation of these lifted plans, we derive a new transportation plan, termed expected sliced transport (EST) plans. We prove that using the EST plan to weight the sum of the individual Euclidean costs for moving from one point to another results in a valid metric between the input discrete probability measures. We demonstrate the connection between our approach and the recently proposed min-SWGG, along with illustrative numerical examples that support our theoretical findings.

\end{abstract}

\section{Introduction}

The optimal transport (OT) problem \citep{Villani2009Optimal} seeks the most efficient way to transport a distribution of mass from one configuration to another, minimizing the cost associated with the transportation process. It has found diverse applications in machine learning due to its ability to provide meaningful distances, i.e., the Wasserstein distances \citep{peyre2019computational}, between probability distributions, with applications ranging from supervised learning \citep{frogner2015learning} to generative modeling \citep{arjovsky2017wasserstein}. Beyond merely measuring distances between probability measures, the optimal transportation plan obtained from the OT problem provides correspondences between the empirical samples of the source and target distributions, which are used in various applications, including domain adaptation \citep{courty2014domain}, positive-unlabeled learning \citep{chapel2020partial}, texture mixing \citep{rabin2011wasserstein}, color transfer \citep{rabin2014adaptive}, image analysis \citep{basu2014detecting}, and even single-cell and spatial omics \citep{bunne2024optimal}, to name a few.

One of the primary challenges in applying the OT framework to large-scale problems is its computational complexity. Traditional OT solvers for discrete measures typically scale cubically with the number of samples (i.e., the support size) \citep{Kolouri2017Optimal}. This computational burden has spurred significant research efforts to accelerate OT computations. Various approaches have been developed to address this challenge, including entropic regularization \citep{cuturi2013sinkhorn}, multiscale methods \citep{schmitzer2016sparse}, and projection-based techniques such as sliced-Wasserstein distances \citep{rabin2011wasserstein} and robust subspace OT \citep{paty2019subspace}. Each of these methods has its own advantages and limitations.

For instance, the entropic regularized OT is solved via an iterative algorithm (i.e., the Sinkhorn algorithm) with quadratic computational complexity per iteration. However, the number of iterations required for convergence typically increases as the regularization parameter decreases, which can offset the computational benefits of these methods. Additionally, while entropic regularization interpolates between Maximum-Mean Discrepancy (MMD) \citep{gretton2012kernel} and the Wasserstein distance \citep{feydy2019interpolating}, it does not produce a true metric between probability measures. Despite not providing a metric, the entropic OT provides a transportation plan, i.e., soft correspondences, albeit not the optimal one. On the other hand, sliced-Wasserstein distances offer linearithmic computational complexity, enabling the comparison of discrete measures with millions of samples. These distances are also topologically equivalent to the Wasserstein distance and offer statistical advantages, such as better sample complexity \citep{nadjahi2020statistical}. However, despite their computational efficiency, the sliced-Wasserstein approaches do not provide a transportation plan between the input probability measures, limiting their applicability to problems that require explicit coupling between measures.

In this paper, we address two central questions: First, \textbf{can a transportation plan be constructed between two probability measures using the sliced transport framework?} If so, \textbf{can the resulting transportation plan be used to define a metric between the two probability measures?} Within the sliced transport framework, the "slices" refer to the one-dimensional marginals of the source and target probability measures, for which an optimal transportation plan is computed. Crucially, this optimal transportation plan applies to the marginals (i.e., one-dimensional probability measures) rather than the original measures. To derive a transportation plan between the source and target measures, this optimal plan for the marginals must be "lifted" back to the original space.

For discrete measures with equal support size, $N$, and uniform mass, $1/N$, the optimal transportation plan between marginals is represented by a correspondence matrix, specifically an $N \times N$ permutation matrix. Previous works have used directly the correspondence matrix obtained for a slice as a transportation plan in the original space of measures \citep{rowland2019orthogonal, mahey2023fast}. This paper provides a holistic and rigorous analysis of this problem for general discrete probability measures. %\mt{[do we want to say ``general measures'' as we stay in the countably discrete case?]}

Our specific contributions in this paper are: 
\begin{enumerate}
    \item Introducing a computationally efficient transportation plan between discrete probability measures, the Expected Sliced Transport plan.  %\textcolor{red}{``Cheap'' transportation plan for height dimensional probability measures.} Computational complexity:  Given two finite discrete probability measures on $\mathbb R^d$ concentrated at $n$ and $m$ particles, respectively, the complexity of using the linear programming approach for solving the optimal transport problem between them is $\mathcal{O}((n)^3\log(n))$. When using entropic regularization, the problem becomes easier to solve using iterative algorithms like Sinkhorn's algorithm, having complexity $\mathcal{O}(nm\log(n+m))$, but introducing an approximation to the original problem. When dealing with measures in the real line, the computational complexity of Sinkhorn's algorithm is still $\mathcal{O}(nm\log(n+m))$ while (using a quadratic cost function $|u-v|^2$) using dynamic programming the computation complexity is $\mathcal{O}(nm)$ or even better when using specialized sorting algorithms, reaching a complexity of order $\mathcal{O}(n\log(n)+m\log(m))$. The proposed method has complexity $\mathcal{O}(L(n\log(n)+m\log(m)))$ where $L$ is the number of slices or unit vectors in $\mathbb{S}^{d-1}$ considered for computing the proposed EST plan $\bar \gamma$.
    \item  Providing a distance for discrete probability measures, the Expected Sliced Transport (EST) distance. 
    \item Offering both a theoretical proof and an experimental visualization showing that the EST distance is equivalent to the Wasserstein distance (and to weak$^*$ convergence) when applied to discrete measures.
    \item Demonstrating the performance of the proposed distance and the transportation plan in diverse applications, namely interpolation and classification. 
\end{enumerate}

%\textcolor{red}{Paragraph about OT-like metrics and their applications to ML.} 

%\textcolor{red}{OT computation (expensive). Alternatives: Sinkhorn alg. Linear OT. Slicing.}

%\textcolor{red}{Expand on the Slicing idea and introduce our Expected Sliced Transform method.}

% \section{Background}

\section{Expected Sliced Transport}

\subsection{Preliminaries}
%Let us consider cost functions of the form $c(x,y)=h(\|x-y\|)$ for $x,y\in\mathbb{R}^d$ and $C(u,v)=h(|u-v|)$ for $u,v\in\mathbb{R}$, where $h:\mathbb{R}_{\geq 0}\to\mathbb{R}_{\geq 0}$ a convex and increasing function with $h(0)=0$, and $\|\cdot\|$ denotes the Euclidean norm in $\mathbb{R}^d$ \ak{we never use this $h$ anywhere. It is just taking space for no reason.}. 
Given a probability measure $\mu\in \mathcal{P}(\mathbb{R}^d)$  and a unit vector  $\theta\in \mathbb{S}^{d-1}\subset \mathbb{R}^d$, we define  $\theta_\#\mu:=\langle \theta,\cdot\rangle_\#\mu$ to be the \textit{$\theta$-slice} of the measure $\mu$, where $\langle \theta,x\rangle=\theta\cdot x=\theta^T x$ denotes the standard inner product in $\mathbb{R}^d$.  For any pair of probability measures with finite $p$-moment ($p> 1$) $\mu^1,\mu^2\in \mathcal{P}_p(\mathbb{R}^d)$, one can pose the following two Optimal Transport (OT) problems:
On the one hand, consider the classical OT problem, which gives rise to the $p$-Wasserstein metric:
\begin{align}
W_p(\mu^1,\mu^2):=\min_{\gamma\in\Gamma(\mu^1,\mu^2)}\left(\int_{\mathbb R^d\times \mathbb R^d}\|x-y\|^pd\gamma(x,y)\right)^{1/p}\label{eq:OT_usual}
\end{align}
%\ak{Why not replace OT with $W_p(\mu^1,\mu^2)^p$ in the above equation?}
where $\|\cdot\|$ denotes the Euclidean norm in $\mathbb R^d$ and  $\Gamma(\mu^1,\mu^2)\subset \mathcal{P}(\mathbb R^d\times \mathbb R^d)$ is the subset of all probability measures with marginals $\mu^1$ and $\mu^2$. %\ak{in the books that I know the notation used is $\Pi(\mu^1,\mu^2)$ }
%When $c(x,y)=\|x-y\|^p$ for $p\geq 1$, $$W_p(\mu^1,\mu^2):=OT(\mu^1,\mu^2)^{1/p}$$ is the so-called $p$-Wasserstein distance between $\mu^1$ and $\mu^2$.
On the other hand, for a given $\theta\in\mathbb{S}^{d-1}$, consider
the one-dimensional OT problem: 
\begin{align}
W_p(\theta_\#\mu^1,\theta_\#\mu^2)=
\min_{\Lambda_\theta\in\Gamma(\theta_\#\mu^1,\theta_\#\mu^2)}\left(\int_{\mathbb R\times \mathbb R}|u-v|^p d\Lambda_\theta(u,v)\right)^{1/p}\label{eq:OT_theta}
\end{align}
In this case, %if $h$ is additionally continuous and strictly convex function, then 
since the measures $\theta_\#\mu^1,\theta_\#\mu^2$ can be regarded as one-dimensional probabilities in $\mathcal{P}(\mathbb R)$, there exists a unique optimal transport plan, which we denote by $\Lambda_\theta^{\mu_1,\mu_2}$ (see, for e.g.,  \cite[Thm. 2.18, Remark 2.19]{villani2021topics}], \cite[Thm. 16.1]{maggi2023optimal}).%, or \cite[Prop. 2.3, Remark 2.4]{thorpe2018introduction}). \mt{[I actually don't like to cite my notes (I think there are better sources) so I would prefer to remove the citation here (there is also a better version of my notes here:https://drive.google.com/file/d/1-6FpiqPG4GVDEBk1qgsGU141oVKyxVG-/view?usp=sharing). ]} %\footnote{One could consider more general cost functions of the form $c(x,y)=h(\|x-y\|)$ for $x,y\in\mathbb{R}^d$ and $C(u,v)=h(|u-v|)$ for $u,v\in\mathbb{R}$, where $h:\mathbb{R}_{\geq 0}\to\mathbb{R}_{\geq 0}$ a striclty convex and increasing function with $h(0)=0$, and $\|\cdot\|$ denotes the Euclidean norm in $\mathbb{R}^d$}

% \citep{rowland2019orthogonal}}

\subsection{On slicing and lifting transport plans}

In this section, given discrete probability measures $\mu^1,\mu^2\in \mathcal{P}(\mathbb R^d)$, we describe the process of slicing them according to a direction $\theta\in\mathbb S^{d-1}$ and lifting  the optimal transportation plan $\Lambda_\theta^{\mu^1,\mu^2}$, which solves the 1-dimensional OT problem \eqref{eq:OT_theta}, to get a plan in $\Gamma(\mu^1,\mu^2)$. Thus, we obtain a new measure, denoted as $\gamma_\theta^{\mu^1,\mu^2}$, in $\mathcal{P}(\mathbb R^d\times \mathbb R^d)$ with first and second marginal $\mu^1$ and $\mu^2$, respectively. For clarity, we first describe the process for discrete uniform measures and then extend it to any pair of discrete measures.  %\ak{$h$ is used here but I don't think we need to have two sentences at the beginning for that. We can just say that for our $c(x,y)=\|x-y\|^p$ there exists a unique solutions and reference these people. }

\subsubsection{On slicing and lifting transport plans for \textbf{uniform discrete measures}}

Given $N\in \mathbb{N}$, consider the space $\mathcal{P}_{(N)}(\mathbb{R}^d)$ of uniform discrete probability measures concentrated at $N$ particles in $\mathbb{R}^d$, that is,
$$\mathcal{P}_{(N)}(\mathbb{R}^d)=\left\{\frac{1}{N}\sum_{i=1}^N \delta_{x_i}~|~x_i\in\mathbb{R}^d,~\forall i\in\{1,...,N\}\right\}.$$
Let $\mu^1=\frac{1}{N}\sum_{i=1}^N \delta_{x_i},\mu^2=\frac{1}{N}\sum_{j=1}^N \delta_{y_j}\in \mathcal{P}_{(N)}(\mathbb{R}^d)$, where $x_i,y_j\in \mathbb{R}^d$ and $\delta_{x_i}$ denotes a Dirac measure located at $x_i$ (respectively for $\delta_{y_j}$). Let us denote by $\mathcal{U}(\mathbb{S}^{d-1})$ the uniform measure on the hypersphere $\mathbb{S}^{d-1}\subset \mathbb{R}^d$. In this case, the $\theta$-slice of $\mu^1$ is represented by $\theta_\#\mu^1=\frac{1}{N}\sum_{i=1}^N \delta_{\theta\cdot x_i}$, %where $\theta\cdot x_i$ denotes the usual inner product on $\mathbb{R}^d$, 
and similarly for $\theta_\# \mu^2$. Let $\mathbf S_N$ denote the symmetric group of all permutations of the elements in the set $[N]:= \{1,\dots,N\}$. Let $\zeta_{\theta},\tau_{\theta}\in \mathbf S_N$ denote the sorted indices of the projected points $\{\theta\cdot x_i\}_{i=1}^N$ and  $\{\theta\cdot y_j\}_{j=1}^N$, respectively, that is, 
\begin{equation}\label{eq: order}
\theta\cdot x_{\zeta_{\theta}^{-1}(1)}\leq \theta\cdot x_{\zeta_{\theta}^{-1}(2)}\leq \dots\leq \theta\cdot x_{\zeta_{\theta}^{-1}(N)}  \text{ and } \theta\cdot y_{\tau_\theta^{-1}(1)}\leq \theta\cdot y_{\tau_\theta^{-1}(2)}\leq \dots\leq \theta\cdot y_{\tau_\theta^{-1}(N)}    
\end{equation}
The optimal matching from $\theta_\#\mu^1$ to $\theta_\#\mu^2$ for the problem \eqref{eq:OT_theta} is induced by the assignment 
 \begin{equation}\label{eq: 1d asignment}
     \theta\cdot x_{\zeta_\theta^{-1}(i)}\longmapsto \theta\cdot y_{\tau_\theta^{-1}(i)}, \qquad \forall 1\leq i\leq N. 
 \end{equation}
We define $T_\theta^{\mu^1,\mu^2}:\{x_1,\dots, x_N\}\to\{y_1,\dots, y_N\}$ the \textit{lifted transport map} between $\mu^1$ and $\mu^2$ by: 
\begin{equation}\label{eq: unifom case map}
    T_\theta^{\mu^1,\mu^2} (x_i)= y_{\tau^{-1}_\theta(\zeta_\theta(i))}, \qquad \forall 1\leq i\leq N.
\end{equation}
Rigorously,  $T_{\theta}^{\mu^1,\mu^2}$ is not necessarily a function defined on $\{x_1,\dots, x_N\}$ but on the the labels $\{1,\dots, N\}$, as two projected points $\theta\cdot x_i$ and $\theta \cdot x_j$ could coincide for $i\not = j$. 
% More generally, in the definition of $\mu^1=\frac{1}{n}\sum_{i=1}^n\delta_{x_i}\in\mathcal{P}_{(n)}(\mathbb{R}^d)$ one can allow $x_i=x_j$ for some pairs of indexes $i\not =j$. If we do so,  we are generalizing the space of uniform discrete measures. Let us denote by $\mathcal{P}_{\mathbb Q}(\mathbb{R}^d)$ the set of discrete probability measures with finite support and rational weights (i.e., $\mu\in \mathcal{P}_{\mathbb Q}(\mathbb{R}^d)$ if and only if it is of the form $\mu=\sum_{i=1}^Nq_i\delta_{x_i}$ with $x_i\in\mathbb{R}^d$, $q_i \in\mathbb Q$  $\forall i\in [N]$ for some $N\in \mathbb N$, and $\sum_{i=1}^N q_i=1$ ), we have
% $$\mathcal{P}_{(n)}(\mathbb{R}^d)\subset \mathcal{P}_{\mathbb Q}(\mathbb{R}^d), \qquad \forall n\in\mathbb N.$$
As a result, it is more convenient to work with \textit{lifted transport plans}. Indeed, the matrix $u_\theta^{\mu^1,\mu^2}\in \mathbb{R}^{n\times n}$ given by
\begin{equation}\label{eq: gamma theta for uniform}
u_\theta^{\mu^1,\mu^2}(i,j)=
\begin{cases}
     1/N & \text{if } j=\tau^{-1}_\theta(\zeta_\theta(i)) \\
     0 & \text{otherwise} 
\end{cases}
\end{equation}
encodes the weights of optimal transport plan between $\theta_\#\mu^1$ and $\theta_\#\mu^2$ given by
\begin{equation}\label{eq: lambda for unif}
  \Lambda_\theta^{\mu^1,\mu^2}:=\sum_{i,j}u_\theta^{\mu^1,\mu^2}(i,j)\delta_{(\theta\cdot x_i, \theta\cdot y_j)},  
\end{equation}
% and also it encodes
 as well as the weights of the \textit{lifted transport plan} between the original measures $\mu^1$ and $\mu^2$ according to the $\theta$-slice defined by
 \begin{equation}\label{eq: gamma for unif}
     \gamma_\theta^{\mu^1,\mu^2}:=\sum_{i,j}u_\theta^{\mu^1,\mu^2}(i,j)\delta_{( x_i, y_j)}.
 \end{equation}
This new measure $\gamma_\theta^{\mu^1,\mu^2}\in \mathcal{P}(\mathbb R^d  \times \mathbb R^d)$  has marginals $\mu^1$ and $\mu^2$. 
While $\gamma_\theta^{\mu^1,\mu^2}$ is not necessarily optimal for the OT problem \eqref{eq:OT_usual} between $\mu^1$ and $\mu^2$, it can be interpreted as a transport plan in $\Gamma(\mu^1,\mu^2)$ which is optimal when projecting $\mu^1$ and $\mu^2$ in the direction of $\theta$. See Figure \ref{fig: sliced_transport_1d} (a)
for a visualization.

\begin{figure}[t!]
    \centering
    \includegraphics[width=\linewidth]{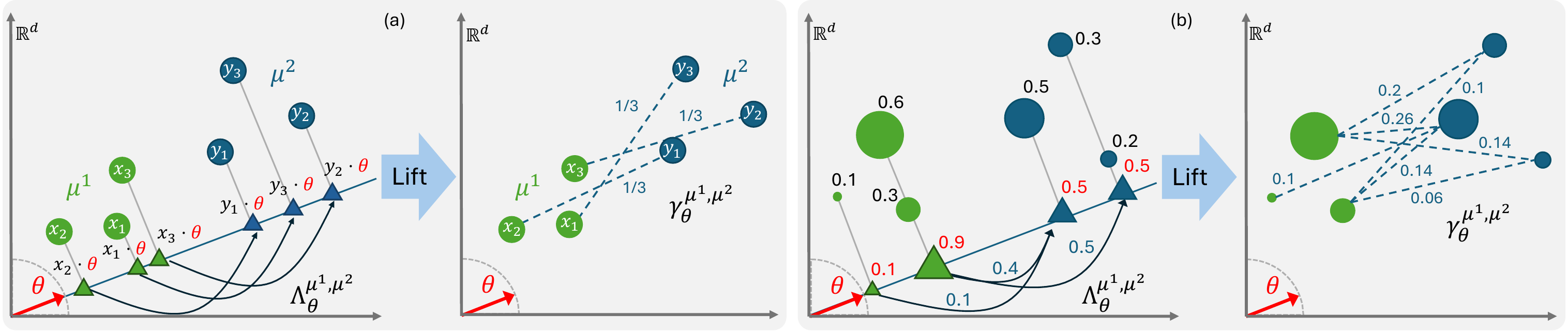}
    \caption{Visualization of the 1-dimensional plan $\Lambda_\theta^{\mu^1,\mu^2}$ (given an unit vector $\theta$) and the corresponding lifted transport plan $\gamma_\theta^{\mu^1,\mu^2}$ between discrete probability measures $\mu^1$ (green circles) and $\mu^2$ (blue circles). In (a) the measures $\mu^1,\mu^2$ are uniform and the masses do not overlap when projecting in the direction of $\theta$. In (b) the measures $\mu^1,\mu^2$ are not uniform and some of the masses overlap when projecting in the direction of $\theta$. For more details see Remark \ref{remark-caption} in Appendix \ref{app: metric prop prelim}.}
    \label{fig: sliced_transport_1d}
\end{figure}

\subsubsection{On slicing and lifting transport plans for \textbf{general discrete measures}}\label{subsec: discrete measures}

Consider discrete measures $\mu^1,\mu^2\in\mathcal{P}(\mathbb{R}^d)$. In this section, we will use the notation $\mu^1=\sum_{x\in\mathbb{R}^d}p(x)\delta_x$,
where $p(x)\geq 0$ for all $x\in \mathbb{R}^d$, $p(x)\not= 0$ for at most countable many points $x\in\mathbb{R}^d$, and $\sum_{x\in\mathbb{R}^d}p(x)=1$. Similarly, $\mu^2=\sum_{y\in\mathbb{R}^d}q(y)\delta_y$ for a  non-negative density function $q$ in $\mathbb R^d$ with finite or countable support and such that $\sum_{y\in\mathbb{R}^d}q(y)=1$. Given $\theta\in\mathbb{S}^{d-1}$, consider the equivalence relation defined by: $$x\sim_\theta x' \quad  \text{ if and only if } \quad \theta\cdot x=\theta\cdot x'$$ 
We denote by $\bar x^\theta $ the equivalence class of $x\in \mathbb{R}^d$. By abuse of notation, we will use interchangeably that $\bar x^\theta$ is a point in the quotient space $\mathbb{R}^d/{\sim_\theta}$, and also the set $\{x'\in\mathbb{R}^d:\ \theta\cdot x= \theta\cdot x'\}$, which is the orthogonal hyperplane to $\theta$ that intersects $x$. The intended meaning of $\bar x^\theta$ will be clear from the context. Notice that, geometrically, the quotient space $\mathbb{R}^d/{\sim_\theta}$ is the line $\mathbb R$ in the direction of $\theta$.

Now, we interpret the projected measures $\theta_\#\mu^1$, $\theta_\#\mu^2$ as 1-dimensional probability measures in $\mathcal{P}(\mathbb{R}^d/{\sim_\theta})$ given by  $\theta_\#\mu^1=\sum_{\bar x^\theta\in \mathbb{R}^d/{\sim_\theta}}P(\bar x^\theta)\delta_{\bar x^\theta}$, where 
$P(\bar x^\theta)=\sum_{x'\in\bar x^\theta}p(x')$, and similarly, $\theta_\#\mu^2=\sum_{\bar y^\theta\in \mathbb{R}^d/{\sim_\theta}}Q(\bar y^\theta)\delta_{\bar y^\theta}$, where 
$Q(\bar y^\theta)=\sum_{y'\in\bar y^\theta}q(y')$.

\begin{remark}
Notice that if $P(\bar x^\theta)=0$, then $p(x')=0$ for all $x'\in \bar x^\theta$, or, equivalently, if $p(x)\not=0$, then $P(\bar x^\theta)\not=0$ (where $x$ is any `representative' of the class $\bar{x}^\theta$). Similarly for $Q$. 
\end{remark}

Consider the optimal transport plan $\Lambda_\theta^{\mu^1,\mu^2}\in \Gamma(\theta_\#\mu^1,\theta_\#\mu^2)\subset \mathcal{P}(\mathbb{R}^d/{\sim_\theta}\times \mathbb{R}^d/{\sim_\theta})$ between 
$\theta_\#\mu^1$ and $\theta_\#\mu^2$, which is \textit{unique} for the OT problem \eqref{eq:OT_theta} as we are considering 1-dimensional probability measures. Let us define
\begin{equation*}
    u_\theta^{\mu^1,\mu^2}(x,y):=\begin{cases}
        \frac{p(x)q(y)}{P(\bar x^\theta)Q(\bar y^\theta)}\Lambda_\theta^{\mu^1,\mu^2}(\{(\bar x^\theta,\bar y^\theta)\}) & \text{ if } p(x)\not = 0 \text{ and } q(y)\not=0\\
        0 & \text{ if } p(x) = 0 \text{ or } q(y)=0
    \end{cases}
\end{equation*}
which allows us to generalize the \textit{lifted transport plan} given in \eqref{eq: gamma theta for uniform} in the general discrete case:
\begin{equation}\label{eq: gamma theta for general discrete}    \gamma_\theta^{\mu^1,\mu^2}:=\sum_{x\in\mathbb{R}^d}\sum_{y\in\mathbb{R}^d}u_\theta^{ \mu^1,\mu^2}(x,y)\delta_{(x,y)}
\end{equation}
See Figure \ref{fig: sliced_transport_1d} (b) for a visualization. 

\begin{remark}
 Notice that this \textit{lifting} process can be performed by starting with any transport plan $\Lambda_\theta\in \Gamma(\theta_\#\mu^1,\theta_\#\mu^2)$, but in this article we will always consider the optimal transportation plan, i.e.,  $\Lambda_\theta=\Lambda_\theta^{\mu^1,\mu^2}$. The reason why we make this choice if because this will give rise to a metric between discrete probability measures: The EST distance which will be defined in Section~\ref{subsec:EST:ESTDis}. 
\end{remark}

\begin{lemma}\label{lem: gamma_theta discrete}
    Given general discrete probability measures $\mu^1$ and $\mu^2$ in $\mathbb{R}^d$, 
    the discrete measure $\gamma_\theta^{\mu^1,\mu^2}$ defined by \eqref{eq: gamma theta for general discrete} has marginals $\mu^1$ and $\mu^2$, that is, $\gamma_\theta^{\mu^1,\mu^2}\in\Gamma(\mu^1,\mu^2)\subset\mathcal{P}(\mathbb R^d\times \mathbb R^d)$.
\end{lemma}

We refer the reader to the appendix for its proof.

\subsection{Expected Sliced Transport (EST) for discrete measures} \label{subsec:EST:ESTDis}

Leveraging on the transport plans $\gamma_\theta^{\mu^1,\mu^2}$ described before, in this section we propose a new transportation plan $\bar\gamma^{\mu^1,\mu^2}\in \Gamma(\mu^1,\mu^2)$, 
%which gives a notion of frequency... ??
%\mt{[sentence TBC...]}
%This transportation plan will 
which will give rise to a new metric in the space of discrete probability measures.

\begin{definition}[Expected Sliced Transport plan]
Let $\sigma\in\mathcal{P}(\mathbb{S}^{d-1})$. %with $supp(\sigma)=\mathbb{S}^{d-1}$.
For  discrete probability measures $\mu^1,\mu^2$ in $\mathbb{R}^d$, we define the expected sliced transport plan $\bar\gamma^{\mu^1,\mu^2}\in \mathcal{P}(\mathbb R^d \times \mathbb R^d)$ by  
\begin{equation}\label{eq: bar gamma 2}
    \bar{\gamma}^{\mu^1,\mu^2}:=\mathbb{E}_{\theta\sim \sigma}[\gamma_\theta^{\mu^1,\mu^2}],
\qquad \text{ where each } \gamma_\theta^{\mu^1,\mu^2} \text{ is given by } \eqref{eq: gamma theta for general discrete},
\end{equation}
that is, 
\begin{align*}
    \bar\gamma^{\mu^1,\mu^2}(\{(x,y)\})=\int_{\mathbb S^{d-1}}\gamma_\theta^{\mu^1,\mu^2}(\{(x,y)\})d\sigma(\theta), \qquad \forall x,y\in\mathbb{R}^d\times \mathbb{R}^d.
\end{align*}
In other words, $\bar\gamma^{\mu^1,\mu^2}=\sum_{x\in\mathbb R^d}\sum_{y\in\mathbb R^d}U^{\mu^1,\mu^2}(x,y)\delta_{(x,y)}$, where the new weights are given by
\begin{equation*}
    U^{\mu^1,\mu^2}(x,y)=\begin{cases}
        p(x)q(y)\int_{\mathbb S^{d-1}}\frac{\Lambda_\theta^{\mu^1,\mu^2}(\{(\bar x^\theta,\bar y^\theta)\})}{P(\bar x^\theta)Q(\bar y^\theta)}d\sigma(\theta) &\text{ if } p(x)\not=0 \text{ and } q(y)\not=0\\
        0 &\text{ otherwise}
    \end{cases}
\end{equation*}

\end{definition}

\begin{remark}
The measure $\bar{\gamma}^{\mu^1,\mu^2}$ is well-defined and, moreover, (as an easy consequence of Lemma~\ref{lem: gamma_theta discrete})  $\bar{\gamma}^{\mu^1,\mu^2}\in \Gamma(\mu^1,\mu^2)$, i.e., it has marginals $\mu^1$ and $\mu^2$. (See also Lemma~\ref{lem:barGammaMarginals} in the appendix.)
\end{remark}

\begin{definition}[Expected Sliced Transport distance]
Let $\sigma\in\mathcal{P}(\mathbb{S}^{d-1})$ with $\supp(\sigma)=\mathbb{S}^{d-1}$.
We define the Expected Sliced Transport discrepancy for discrete probability measures $\mu^1$,$\mu^2$ in $\mathbb R^d$ by 
\begin{align}
    \mathcal{D}_p(\mu^1,\mu^2)&:= \left(\sum_{x\in \mathbb{R}^d}\sum_{y\in\mathbb{R}^d} \|x-y\|^p \, \bar{\gamma}^{\mu^1,\mu^2}(\{(x,y)\})\right)^{1/p}, %\qquad \text{ for discrete measures } \mu^1,\mu^2\in\mathcal{P}(\mathbb{R}^d),
    \label{eq:est 2}
\end{align}
where $\bar\gamma^{\mu^1,\mu^2}$ is defined by \eqref{eq: bar gamma 2}.
\end{definition}

\begin{remark}
By defining the following generalization of the Sliced Wasserstein
Generalized Geodesics
(SWGG) dissimilarity presented in \cite{mahey2023fast},   \begin{equation}\label{eq:D_theta}      \mathcal{D}_p(\mu^1,\mu^2;\theta):=\left(\sum_{x\in \mathbb{R}^d}\sum_{y\in\mathbb{R}^d}\|x-y\|^p\gamma_\theta^{\mu^1,\mu^2}(\{x,y\})\right)^{1/p},
    \end{equation}
    we can rewrite \eqref{eq:est 2} as
    $$ \mathcal{D}_p(\mu^1,\mu^2)=\mathbb{E}^{1/p}_{\theta\sim\sigma}[\mathcal{D}^p_p(\mu^1,\mu^2;\theta)]$$
\end{remark}

\begin{remark}\label{remark: W<D}
Since the EST plan $\bar\gamma^{\mu^1,\mu^2}$ is a transportation plan, we have that   $$W_p(\mu^1,\mu^2)\leq \mathcal{D}_p(\mu^1,\mu^2).$$ In Appendix \ref{app: weak*} we will show that they define the same topology in the space of discrete probability measures. 
\end{remark}

\begin{remark}[EST for discrete uniform measures and the Projected Wasserstein distance]\label{remark: PWD} Consider uniform measures $\mu^1=\frac{1}{N}\sum_{i=1}^N \delta_{x_i},\mu^2=\frac{1}{N}\sum_{j=1}^N \delta_{y_j}\in \mathcal{P}_{(N)}(\mathbb{R}^d)$, and for $\theta\in\mathbb S^{d-1}$, let $\zeta_\theta,\tau_\theta\in \mathbf S_N$ be permutations that allow us to order the projected points as in \eqref{eq: order}.
 Notice that if $\sigma=\mathcal{U}(\mathbb{S}^{d-1})$, by using the formula \eqref{eq: unifom case map} for each assignment given $\theta$ and noticing that  $\tau_\theta^{-1}\circ\zeta_\theta\in\mathbf S_N$,  we can re-write \eqref{eq:est 2} as
 \begin{equation}\label{eq: PWD}
     \mathcal{D}_p(\mu^1,\mu^2)^p=  \mathbb{E}_{\theta\sim \mathcal{U}(\mathbb{S}^{d-1})}\left[\frac{1}{N}\sum_{i=1}^N \|x_i-y_{\tau^{-1}_{{\theta}}(\zeta_\theta(i))}\|^p\right].
 \end{equation}
 Therefore, %when $c(x,y)=\|x-y\|^p$, for $p\geq 1$, 
 we have that the expression for $\mathcal{D}_p(\cdot,\cdot)$ given by \eqref{eq: PWD} coincides with the \textbf{Projected Wasserstein distance} proposed in \cite[Definition 3.1]{rowland2019orthogonal}. Then, by applying \cite[Proposition 3.3]{rowland2019orthogonal}, we have that the Expected Sliced Transport discrepancy defined in 
Equation \eqref{eq: PWD} is a metric on the space $\mathcal{P}_{(N)}(\mathbb{R}^d)$.
We generalise this in the next theorem.
\end{remark}

\begin{theorem}\label{thm: metric discrete}
%Let $c(x,y)=\|x-y\|^p$, for $1\leq p<\infty$.
The Expected Sliced Transport $\mathcal{D}_p(\cdot,\cdot)$ defined in \eqref{eq:est 2} is a metric in the space of finite discrete probability measures in $\mathbb{R}^d$. 
\end{theorem}

\begin{proof}[Sketch of the proof of Theorem \ref {thm: metric discrete}]
For the detailed proof, we refer the reader to Appendix \ref{app: metric discrete}. Here, we present a brief overview of the main ideas and steps involved in the proof.

The symmetry of $\mathcal{D}_p(\cdot,\cdot)$ follows from our construction of the transport plan $\bar\gamma^{\mu^1,\mu^2}$, which is based on considering a family of \textit{optimal} 1-d transport plans $\{\Lambda_\theta^{\mu^1,\mu^2}\}_{\theta\in\mathbb S^{d-1}}$. 
The identity of indiscernibles follows essentially from Remark \ref{remark: W<D}. 
To prove the triangle inequality we use the following strategy: 
\begin {enumerate}
%\item We use the 1-d transport transport plan $\Lambda_\theta^{\mu^1,\mu^2}$ defined by \eqref {eq: gamma theta for uniform} and show that the corresponding lifted plan $\gamma_\theta^{\mu^1,\mu^2}$ in \eqref{eq: gamma for unif} belongs to $\Gamma(\mu^1,\mu^2)$.
\item We leverage the fact that  $\mathcal{D}_p(\cdot,\cdot)$ is a metric for the space $\mathcal{P}_{(N)}(\mathbb{R}^d)$ of uniform discrete probability measures concentrated at $N$ particles in $\mathbb{R}^d$ (\cite{rowland2019orthogonal}) to prove that $\mathcal{D}_p(\cdot,\cdot)$ is also a metric on the set of  measures in which the masses are rationals. To do so, we establish a correspondence between  finite discrete measures with rational weights and finite discrete measures with uniform mass (see the last paragraph of Proposition \ref{prop: rational metric}).
\item Given finite discrete measures $\mu^1,\mu^2$, we approximate them, in terms of the  Total Variation norm, by sequences of probability measures with rational weights  $\{\mu^1_n\},\{\mu^2_n\},$ supported on the same points as $\mu^1$ and $\mu^2$, respectively.
 %$$\{\widetilde{\mu^1_n}\},\{\widetilde{\mu^2_n}\}$$
We then turn our attention on how the various plans constructed behave as the $n$ increases and show the following convergence results in Total Variation norm:
\begin{enumerate}
\item     The sequence $\left(\Lambda_\theta^{\mu^1_n,\mu_n^2}\right)_{n\in \mathbb N}$ converges to $\Lambda_\theta^{\mu^1,\mu^2}$.
\item
    The sequence  $\left(\gamma_\theta^{\mu^1_n,\mu_n^2}\right)_{n\in  \mathbb N}$ converges to $\gamma_\theta^{\mu^1,\mu^2}$.
\item
    The sequence  $\left(\bar\gamma^{\mu_n^1,\mu_n^2}\right)_{n\in \mathbb{N}}$ converges to $\bar\gamma^{\mu^1,\mu^2}$.     
\end{enumerate}
As a consequence, we obtain 
$\lim_{n\to\infty}\mathcal{D}_p(\mu_n^1,\mu_n^2)=\mathcal{D}_p(\mu^1,\mu^2).$
\end{enumerate}
   Finally, given  three finite discrete measures $\mu^1,\mu^2,\mu^3$, we proceed as in point 2 by considering sequences of probability measures with rational weights  $\{\mu^1_n\},\{\mu^2_n\}, \{\mu^3_n\}$ supported on the same points as $\mu^1$, $\mu^2$, $\mu^3$, respectively, that approximate the original measures in Total Variation, obtaining
   \begin{equation*}
       \mathcal{D}_p(\mu^1,\mu_n^2)=\lim_{n\to\infty}\mathcal{D}_p(\mu_n^1,\mu_n^2)\leq\lim_{n\to\infty} \mathcal{D}_p(\mu_n^1,\mu_n^3)+\mathcal{D}_p(\mu^3_n,\mu_n^2)=\mathcal{D}_p(\mu^1,\mu^3)+\mathcal{D}_p(\mu^3,\mu^2)
   \end{equation*}
   where the equalities follows from point 2 and the middle triangle inequality follows from point 1. \qedhere
\end{proof}

\section{Experiments}

\begin{figure}[t!]
    \centering
    \includegraphics[width=\linewidth]{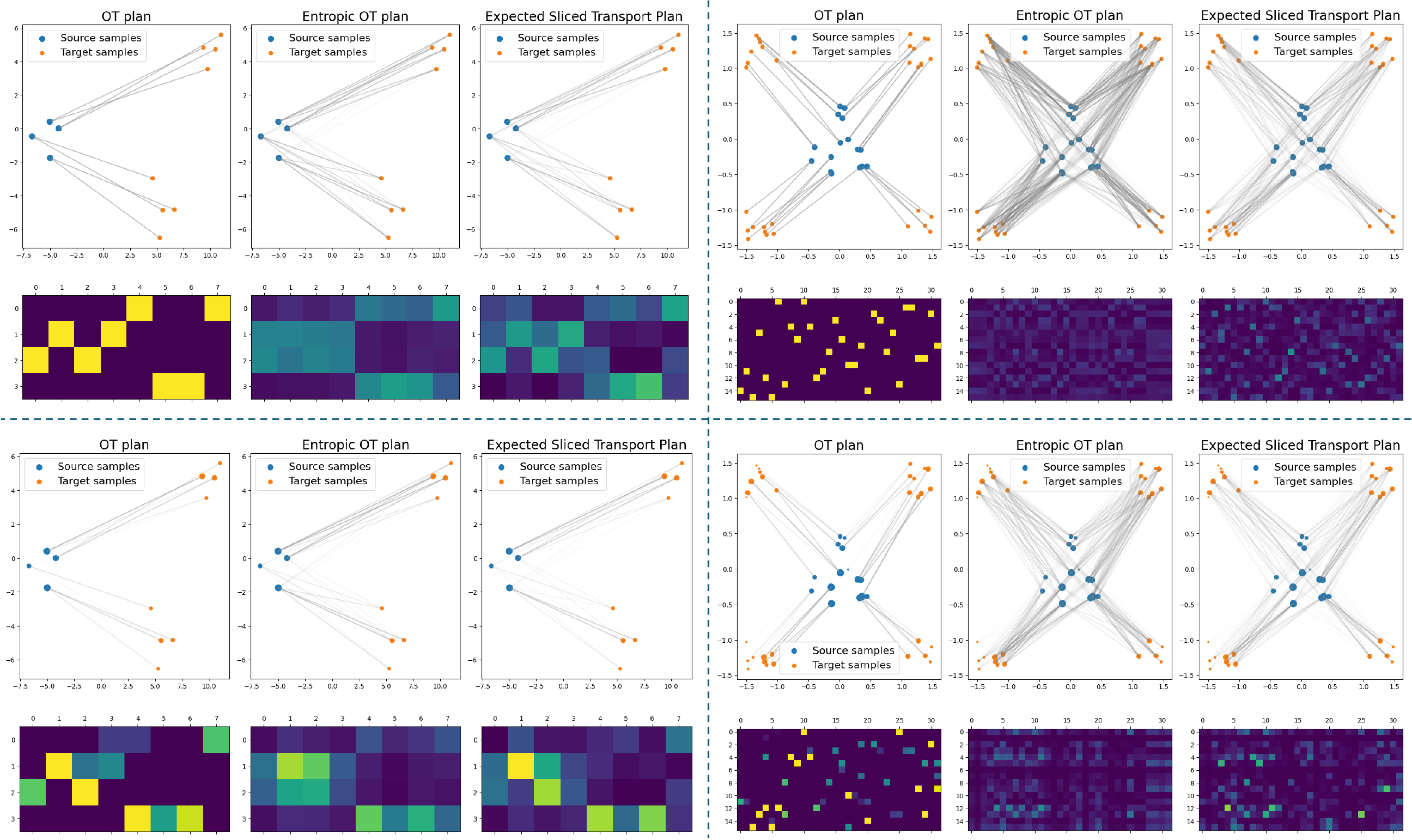}
    \caption{Depiction of transport plans (an optimal transport plan, a plan obtained from solving an entropically regularized transport problem, and the proposed expected sliced transport plan) between source (orange) and target (blue) for four different configurations of masses. The measures in the left and right panels are concentrated on the same particles, respectively; however, the top row depicts measures with uniform mass, while the bottom row depicts measures with random, non-uniform mass. Transportation plans are shown as gray assignments and as $n\times m$ heat matrices encoding the amount of mass transported (dark color = no transportation, bright color = more transportation), where $n$ is the number of particles on which the source measure is concentrated, and $m=2n$) is the number of particles on which the target measure is concentrated.}
    \label{fig:expected_transportation_plan}
    \vspace{-.1in}
\end{figure}

\subsection{Comparison of Transportation Plans}

Figure \ref{fig:expected_transportation_plan} illustrates the behavior of different transport schemes: the optimal transport plan for $W_2(\cdot,\cdot)$, the transport plan obtained by solving an entropically regularized transportation problem between the source and target probability measures, and the new expected sliced transport plan $\bar{\gamma}$.
We include comparisons with entropic regularization because it is one of the most popular approaches, as it allows for the use of Sinkhorn's algorithm. From the figure, we observe that while $\bar{\gamma}$ promotes mass splitting, this phenomenon is less pronounced than in the entropically regularized OT scheme. This observation will be revisited in Subsection  \ref{subse:temp}.

\begin{figure}[t!]
\vspace{-.2in}
    \centering
    \includegraphics[width=\linewidth]{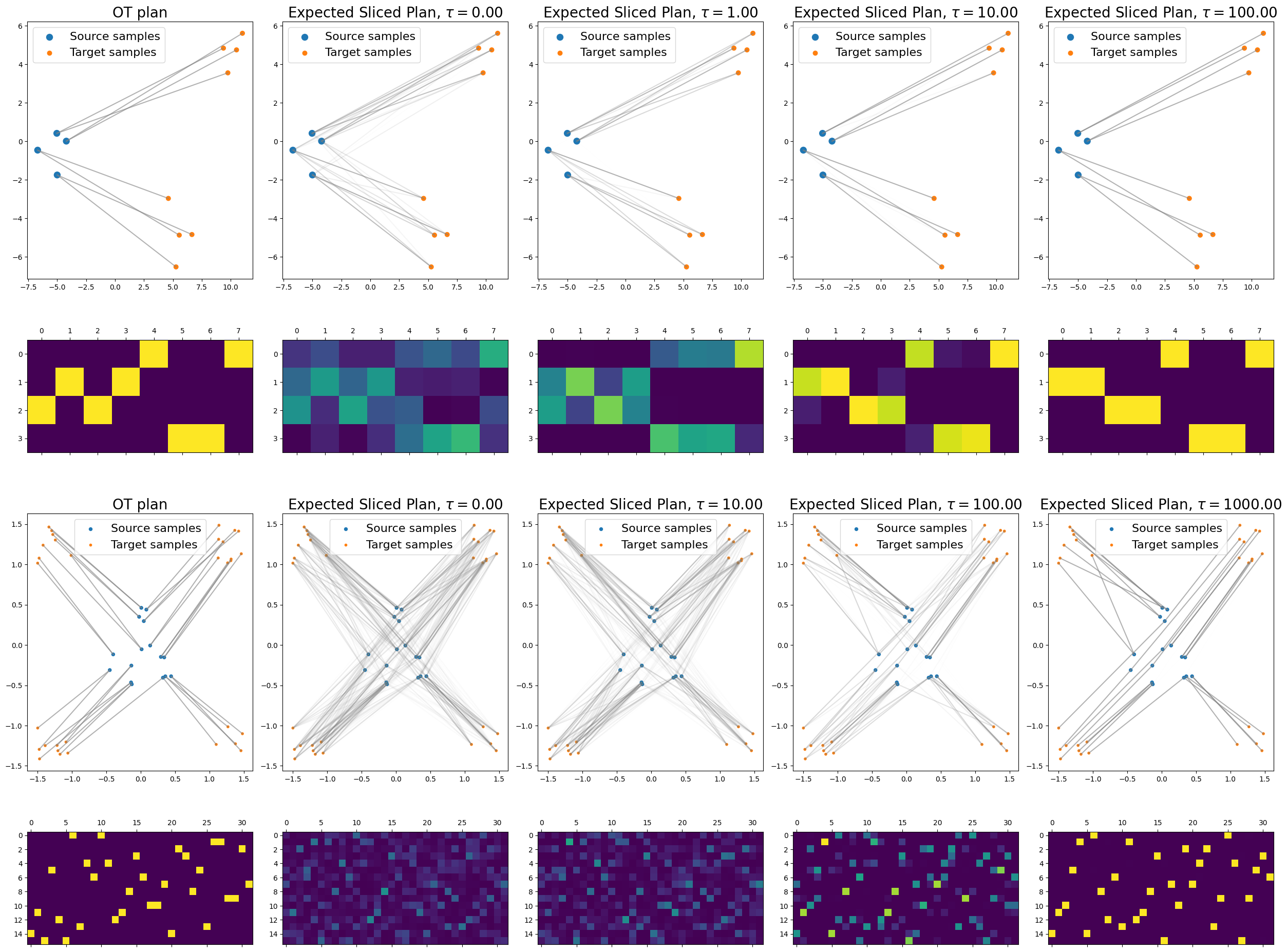}
    \caption{The effect of increasing $\tau$ (i.e., decreasing temperature) on the expected sliced plan. The left most column shows the OT plan, and the rest of the columns show the expected sliced plan as a function of increasing $\tau$. The right most column depicts that expected sliced plan recovers the min-SWGG \cite{mahey2023fast} transportation map.}
    \label{fig:temperature}
    \vspace{-.1in}
\end{figure}

\begin{figure}[t!]
    %\vspace{-.1in}
    \centering
    \includegraphics[width=\linewidth]{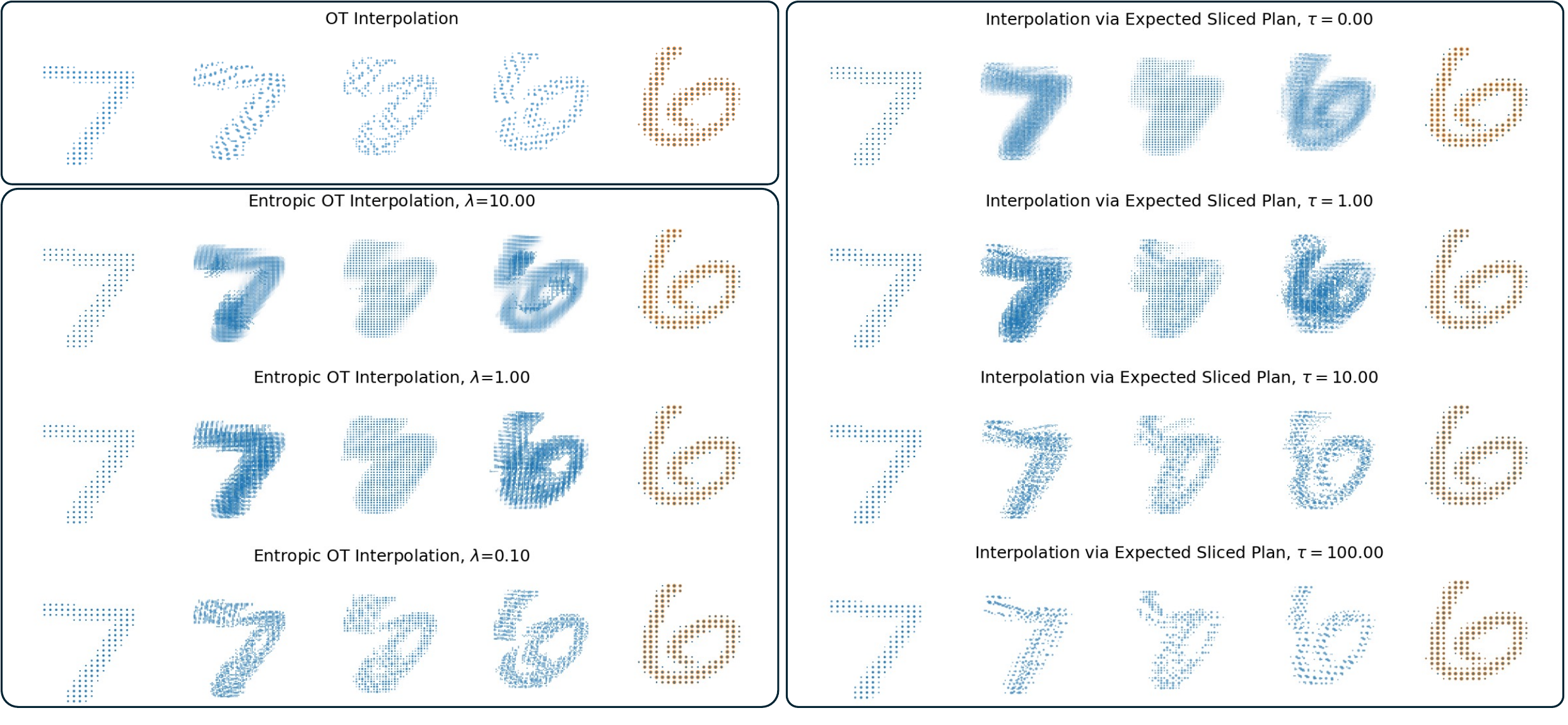}
    \caption{Interpolation between two point clouds via $((1-t)x+ty)_\# \gamma$, where $\gamma$ is the optimal transportation plan for $W_2(\cdot,\cdot)$ (top left), the transportation plan obtained from entropic OT with various regularization parameters (bottom left), and the expected sliced transport plans for different temperatures $\tau$ (right).}
    \label{fig:interpolation}
    \vspace{-.1in}
\end{figure}

\subsection{Temperature approach}\label{subse:temp}

Given $\mu^1,\mu^2$ discrete probability measures, we perform the new expected sliced transportation scheme by using the following averaging measure $\sigma_\tau\ll \mathcal{U}(\mathbb{S}^{d-1})$ on the sphere: %\mt{[We used $\mathcal{U}$ for the uniform measure earlier so I changed it here and below to]}
    \begin{equation}\label{eq:temp}
        d\sigma_\tau(\theta)= \frac{e^{-\tau \mathcal{D}^p_p(\mu^1,\mu^2;\theta)}}{\int_{\mathbb S^{d-1}} e^{-\tau \mathcal{D}_p^p(\mu^1,\mu^2;\theta') }d\theta'},
    \end{equation}
% where $$\mathcal{D}_p(\mu^1,\mu^2;\theta):=\left(\sum_{x,y}\|x-y\|^p\, \gamma_\theta^{\mu^1,\mu^2}(\{(x,y)\})\right)^{1/p},$$
% and 
where $\mathcal{D}_p(\mu^1,\mu^2;\theta)$ is given by \eqref{eq:D_theta}, and $\tau \geq 0$ is a hyperparameter we will refer to as the \textit{temperature}  (note that increasing $\tau$ corresponds to reducing the temperature). If $\tau = 0$, then $\sigma_0 = \mathcal{U} (\mathbb{S}^{d-1})$. However, when $\tau \neq 0$, $\sigma_\tau$ is a probability measure on $\mathbb{S}^{d-1}$ with density given by \eqref{eq:temp}, which depends on the source and target measures $\mu^1$ and $\mu^2$. We have chosen this measure $\sigma_\tau$ because it provides a general parametric framework that interpolates between our proposed scheme with the uniform measure ($\tau = 0$) and min-SWGG \citep{mahey2023fast}, as the EST distance approaches min-SWGG when $\tau \to \infty$. For the implementations, we use
\begin{equation}\label{eq:temp_implement}
\sigma_\tau(\theta^l)=\frac{e^{-\tau \mathcal{D}_p^p(\mu^1,\mu^2;\theta^l)}}{\sum_{\ell'=1}^Le^{-\tau \mathcal{D}_p^p(\mu^1,\mu^2;\theta^{\ell'})}},    
\end{equation}
where $L$ represents the number of slices or unit vectors $\theta^1, \dots, \theta^L \in \mathbb{S}^{d-1}$. Figure \ref{fig:temperature}
 illustrates that as $\tau \to \infty$, the weights used for averaging the lifted transportation plans converge to a one-hot vector, i.e., the slice minimizing $\mathcal D_p(\mu^1, \mu^2; \theta)$ dominates, leading to a transport plan with fewer mass splits. For the visualization we have used source $\mu^1$ and target $\mu^2$ uniform probability measures concentrated on different number of particles. For consistency, the configurations are the same as in Figure \ref{fig:expected_transportation_plan}.

\subsection{Interpolation}
We use the Point Cloud MNIST 2D dataset \citep{Garcia2023PointCloudMNIST2D}, a reimagined version of the classic MNIST dataset \citep{lecun1998mnist}, where each image is represented as a set of weighted 2D point clouds instead of pixel values. In Figure \ref{fig:interpolation}, we illustrate the interpolation between two point clouds that represent digits 7 and 6. Since the point clouds are discrete probability measures with non-uniform mass, we perform three different interpolation schemes via $((1-t)x + ty)_\#\gamma$ where $0 \leq t \leq 1$ for different transportation plans $\gamma$, namely:
% (1) $\gamma = \gamma^*$ an optimal transportation plan for $W_2(\cdot, \cdot)$; (2)
% a transportation plan $\gamma$ obtained from solving an entropically regularized transportation problem (performed for three different regularization parameters $\lambda$);
% (3) $\gamma = \bar{\gamma}$: the expected sliced transport plan computed using $\sigma_\tau$ given by formula \eqref{eq:temp} (or \eqref{eq:temp_implement} for implementations) for four different values of the temperature parameter $\tau$.
\medskip

\begin{enumerate}
    \item $\gamma = \gamma^*$, an optimal transportation plan for $W_2(\cdot, \cdot)$; 
    \item a transportation plan $\gamma$ obtained from solving an entropically regularized transportation problem (performed for three different regularization parameters $\lambda$);
\item  $\gamma = \bar{\gamma}$: the expected sliced transport plan computed using $\sigma_\tau$ given by formula \eqref{eq:temp} (or \eqref{eq:temp_implement} for implementations) for four different values of the temperature parameter $\tau$.
\end{enumerate}

As the temperature increases, the transportation plan exhibits less mass splitting, similar to the effect of decreasing the regularization parameter $\lambda$ in entropic OT. However, unlike entropic OT, where smaller regularization parameters require more iterations for convergence, the computation time for expected sliced transport remains unaffected by changes in temperature.

% We use the point-cloud MNIST 2D dataset \cite{Garcia2023PointCloudMNIST2D}, which  is a reimagined version of the classic MNIST dataset \cite{lecun1998mnist}, where each image is represented as a set of weighted 2D point clouds instead of pixel values.  In Figure \ref{fig:interpolation}, we show interplolation between two point clouds representing digits 7 and 6. Note that the point clouds are discrete probability measures with non-uniform mass, we perform three different interpolation schemes
% $$((1-t)x+ty)_\#\gamma, \qquad 0\leq t\leq 1$$ 
% for different transportation plans, namely:
% \begin{enumerate}
%     \item $\gamma=\gamma^*$ an optimal transportation plan for $W_2(\cdot, \cdot)$;
%     \item a transportation plan $\gamma$ obtained from solving an entropically regularized transportation problem (this is performed for three different regularization parameters $\lambda$);
%     \item $\gamma=\bar\gamma$ the expected sliced transport plan computed with $\sigma_\tau$ given by formula \eqref{eq:temp} (or \eqref{eq:temp_implement} for the implementations) for four different values for the temperature parameter $\tau$.  
% \end{enumerate}

\begin{wrapfigure}[23]{r}{0.4\linewidth}   
\vspace{-.2in}
    \includegraphics[width=\linewidth]{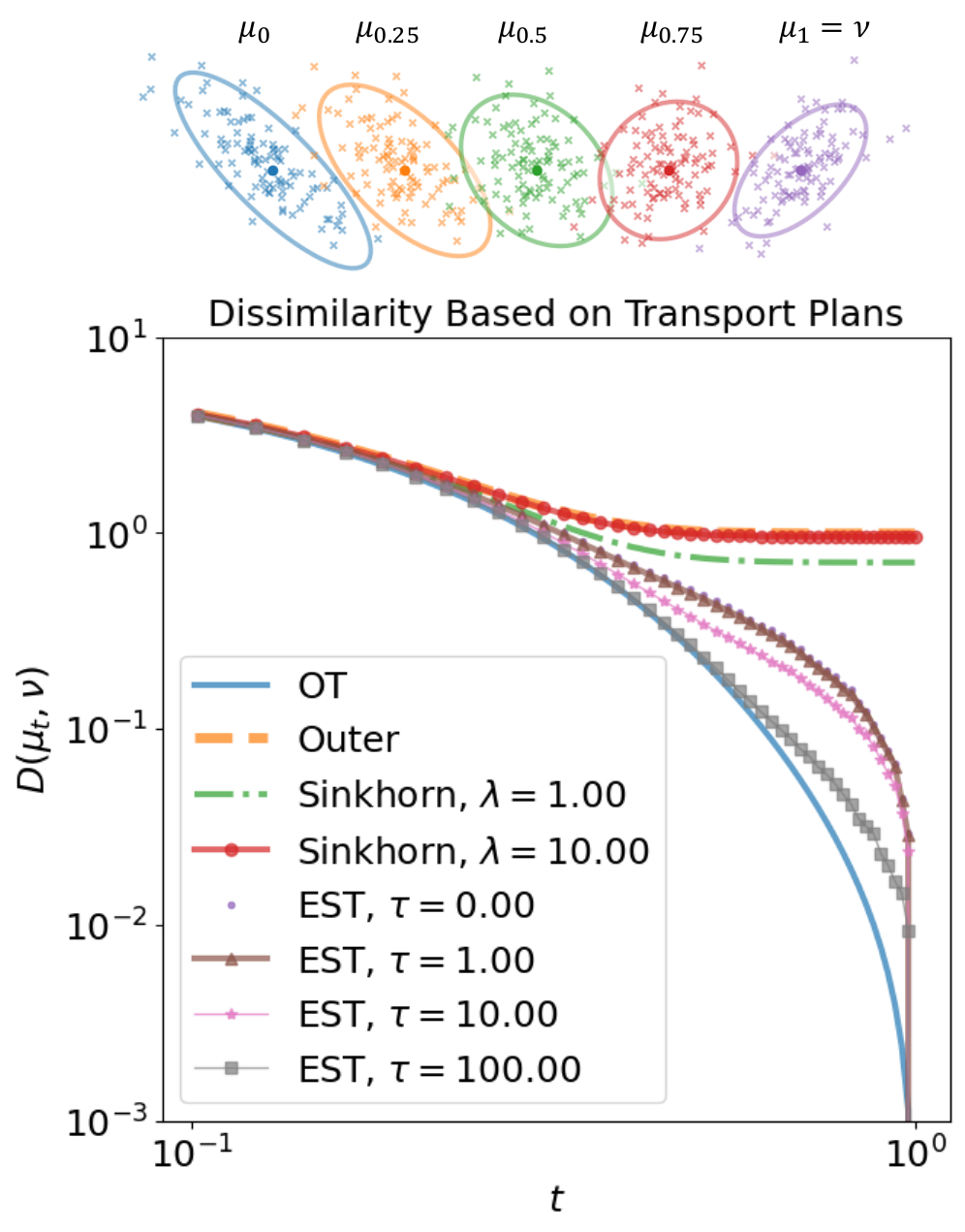}
    \vspace{-.3in}
    \caption{Discrepancies calculated from transportation plans between $\mu_t$ and $\nu$, when $\mu_t\weakstarto \nu$, as a function of $t\in [0,1]$.}
    \label{fig:weak_conv}
\end{wrapfigure}

\newpage
\subsection{Weak Convergence}
Given finite discrete probability measures $\mu$ and $\nu$, we consider $\mu_t$, for $0\leq t\leq 1$,  the Wasserstein geodesic between $\mu$ and $\nu$. In particular, $\mu_t$ is a curve of probability measures that interpolates $\mu$ and $\nu$, that is $\mu_1=\mu$ and $\mu_1=\nu$. Moreover, we have that $W_2(\mu_t,\nu)=(1-t)W_2(\mu,\nu)\longrightarrow 0$ as $t\to 1$, or equivalently, we can say $\mu_t$ converges in the weak$^*$-topology to $\nu$. Figure \ref{fig:weak_conv} illustrates that the expected sliced distance also satisfies $\mathcal{D}_2(\mu_t,\nu)\longrightarrow 0$ as $t\to 1$. Indeed, this experimental conclusion is justified by the following theoretical result: 
%\begin{theorem}

\textit{Let $\mu,\mu_n\in \cP(\Omega)$  be discrete measures with finite or countable support, where $\Omega\subset\bbR^d$ is compact.
Assume $\sigma\ll \cU(\mathbb S^{d-1})$. 
Then, $\cD_p(\mu_n,\mu)\to 0$ if and only if $\mu_n\weakstarto \mu$.}

%\end{theorem}
We present its proof in Appendix \ref{app: weak*}.

For the experiment, $\mu$ and $\nu$ are chosen to be discrete measures with $N$ particles of uniform mass, sampled from two Gaussian distributions (see Figure \ref{fig:weak_conv}, top). %We construct $\mu_t$ by following the OT geodesic between $\mu$ and $\nu$. 
For different values of time, $0\leq t\leq 1$, we compute different discrepancies, $\sum_{i=1}^N \sum_{j=1}^N \| x_i - y_j \|^2 \gamma^{\mu_t, \nu}_{ij}$, calculated for various transport plans: (1) the optimal transport plan, (2) the outer product plan $\mu_t \otimes \nu$, (3) the plan obtained from entropic OT with two different regularization parameters $\lambda$, and (4) our proposed expected sliced plan computed with $\sigma_\tau$ given in \eqref{eq:temp} for four different temperature parameters $\tau$. As $\mu_t$ converges to $\nu$, it is evident that both the OT and our proposed EST distance approach zero, while the entropic OT and outer product plans, as expected, do not converge to zero.

\subsection{Transport-Based Embedding}

Following the linear optimal transportation (LOT) framework, also referred to as the Wasserstein or transport-based embedding framework \citep{wang2013linear,kolouri2021wasserstein,nenna2023transport, bai2023linear,martin2024data}, we investigate the application of our proposed transportation plan in point cloud classification. Let $\mu_0=\sum_{i=1}^N \alpha_i\delta_{x_i}$, denote a reference probability measure, and let $\mu_k = \sum_{j=1}^{N_k} \beta^k_j \delta_{y^k_j}$ denote a target probability measure. Let $\gamma^{\mu_0,\mu_k}$ denote a transportation plan between $\mu_0$ and $\mu_k$, and define the barycentric projection \cite[Definition 5.4.2]{ambrosio2011gradient} of this plan as:
\begin{equation}
    b_i(\gamma^{\mu_0,\mu_k}):= \frac{1}{\alpha_i}\sum_{j=1}^{N_k} \gamma^{\mu_0,\mu_k}_{ij}y^k_j, \qquad i\in{1,...,N}.
\end{equation}
Note that $b_i(\gamma^{\mu_0,\mu_k})$ represents the center of mass to which $x_i$ from the reference measure is transported according to the transportation plan $\gamma^{\mu_0,\mu_k}$. When $\gamma^{\mu_0,\mu_k}$ is the OT plan, the LOT framework of \cite{wang2013linear} uses 
$$
[\phi(\mu_k)]_i := b_i(\gamma^{\mu_0,\mu_k})-x_i, \qquad i\in{1,...,N}
$$
as an embedding $\phi$ for the measure $\mu^k$. This framework, as demonstrated in \cite{kolouri2021wasserstein}, can be used to define a permutation-invariant embedding for sets of features and, more broadly, point clouds. More precisely, given a point cloud $\mathcal{Y}_k=\{(\beta_j^k,y^k_j)\}_{j=1}^{N_k}$, where  $\sum_{j=1}^{N_k}\beta_j^k=1$ and $\beta_j$ represent the mass at location $y_j$, we represent this point cloud as a discrete measure $\mu_k$. 

In this section, we use a reference measure with $N$ particles of uniform mass to embed the digits from the Point Cloud MNIST 2D dataset using various transportation plans. We then perform a logistic regression on the embedded digits and present the results in Figure \ref{fig:tSNE}. The figure shows a 2D t-SNE visualization of the embedded point clouds using: (1) the OT plan, (2) the entropic OT plan with two different regularization parameters, and (3) our expected sliced plan with two temperature parameters (using $N = 100$ for all methods). 
In addition, we report the test accuracy of these embeddings for different reference sizes.
\begin{figure}[t!]
    \centering    \includegraphics[width=\linewidth]{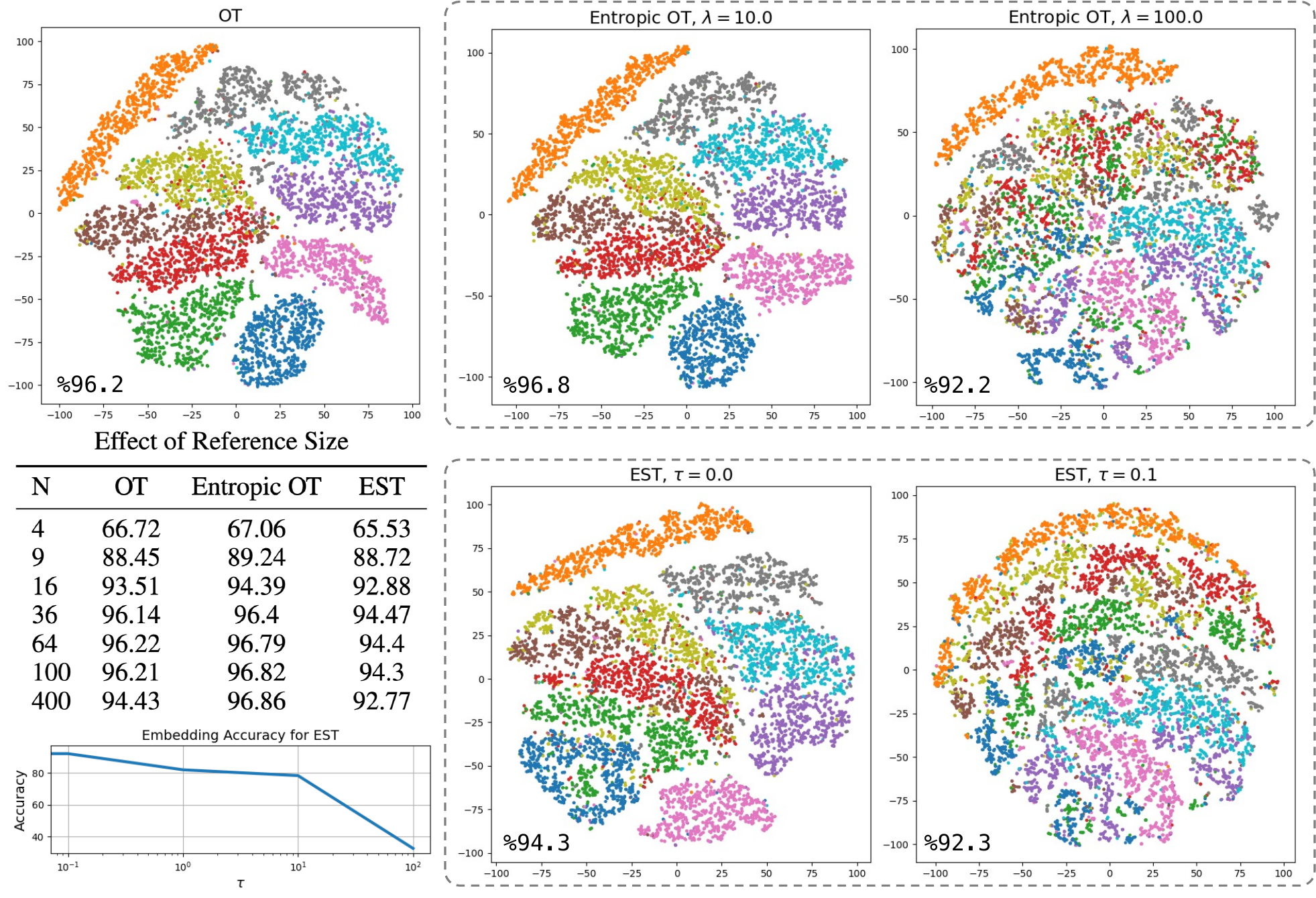}
    \vspace{-.2in}
    \caption{t-SNE visualization of the embeddings computed using different transportation plans, along with the corresponding logistic regression accuracy for each embedding. The t-SNE plots are generated for embeddings with a reference size of $N=100$, and for EST, we used $L=128$ slices. The table shows the accuracy of the embeddings as a function of reference size $N$. For the table, the regularization parameter for entropic OT is set to $\lambda=10$, and for EST, the temperature is set to $\tau=0$ with $L=128$ slices. Lastly, the plot on the bottom left shows the performance of EST, when $N=100$ and $L=128$, as a function of the temperature parameter, $\tau$.}
    \vspace{-.3in}
    \label{fig:tSNE}
\end{figure}

% \begin{wrapfigure}[10]{r}{.3\linewidth}
%     \vspace{-.3in}
%     \includegraphics[width=\linewidth]{./ESP_LOT_tau.png}
%     \vspace{-.3in}
%     \caption{Increasing $\tau$ leads to a drop in the performance when calculating the ESP-based embedding.}
%     \label{fig:ablation}
% \end{wrapfigure}

Lastly, we make an interesting observation about the embedding computed using EST. As we reduce the temperature, i.e., increase $\tau$, the embedding becomes progressively less informative. We attribute this to the dependence of $\sigma_\tau$ on $\mu_k$. In other words, the embedding is computed with respect to different $\sigma_\tau$ for different measures, leading to inaccuracies when comparing the embedded measures. This finding also suggests that the min-SWGG framework, while meritorious, may not be well-suited for defining a transport-based embedding.

% \input{ablation_tables}

% \section{Limitations and future work}

\section{Conclusions}

In this paper, we explored the feasibility of constructing transportation plans between two probability measures using the computationally efficient sliced optimal transport (OT) framework. We introduced the Expected Sliced Transport (EST) framework and proved that it provides a valid metric for comparing discrete probability measures while preserving the computational efficiency of sliced transport and enabling explicit mass coupling. Through a diverse set of numerical experiments, we illustrated the behavior of this newly introduced transportation plan. Additionally, we demonstrated how the temperature parameter in our approach offers a flexible framework that connects our method to the recently proposed min-Sliced Wasserstein Generalized Geodesics (min-SWGG) framework. Finally, the theoretical insights and experimental results presented here open up new avenues for developing efficient transport-based algorithms in machine learning and beyond.

\section*{Acknowledgement}

SK acknowledges support from NSF CAREER Award \#2339898. 
MT was supported by the Leverhulme Trust Research through the Project
Award “Robust Learning: Uncertainty Quantification, Sensitivity and Stability” (grant agreement RPG-2024-051) and the EPSRC Mathematical and Foundations of Artificial Intelligence
Probabilistic AI Hub (grant agreement EP/Y028783/1).

\newpage
\clearpage

\bibliography{ExpST.bib}
\bibliographystyle{iclr2025_conference}

\appendix
% \onecolumn

\newpage
\clearpage
\section{Proof of Theorem \ref{thm: metric discrete}: Metric property of the expected sliced discrepancy for discrete probability measures}\label{app: metric discrete}

\subsection{Preliminaries on Expected Sliced Transportation}\label{app: metric prop prelim}

\begin{remark}[Figure \ref{fig: sliced_transport_1d}]\label{remark-caption} Let us elaborate on explaining Figure \ref{fig: sliced_transport_1d}. 
    (a) Visualization for uniform discrete measures $\mu^1=\frac{1}{3}(\delta_{x_1}+\delta_{x_2}+\delta_{x_3})$ (green circles), $\mu^2=\frac{1}{3}(\delta_{y_1}+\delta_{y_2}+\delta_{y_3})$ (blue circles)  in $\mathcal{P}_{(N)}(\mathbb R^d)$ with $n=3$. Given an angle $\theta$ (red unit vector), when sorting  $\{\theta\cdot x_i\}_{i=1}^3$ and  $\{\theta\cdot y_j\}_{j=1}^3$ we use permutations $\zeta_\theta$ and $\tau_\theta$ given by $\zeta_\theta(2)=1, \zeta_\theta(1)=2, \zeta_\theta(3)=3$, and $\tau_\theta(1)=1$, $\tau_\theta(3)=2$, $\tau_\theta(2)=3$. The optimal transport map between $\theta_\#\mu^1$ (green triangles) and $\theta_\#\mu^2$ (blue triangles) is given by the following assignment: $\theta\cdot x_{\zeta_\theta^{-1}(1)}=\theta\cdot x_{2}\longmapsto
    \theta\cdot y_1=\theta\cdot y_{\tau_\theta^{-1}(1)}$, 
    $\theta\cdot x_{\zeta_\theta^{-1}(2)}=\theta\cdot x_{1}\longmapsto
    \theta\cdot y_3=\theta\cdot y_{\tau_\theta^{-1}(2)}$, 
    $\theta\cdot x_{\zeta_\theta^{-1}(3)}=\theta\cdot x_{3}\longmapsto
    \theta\cdot y_2=\theta\cdot y_{\tau_\theta^{-1}(3)}$. This gives rise to the plan $\Lambda_\theta^{\mu^1,\mu^2}$ given in \eqref{eq: lambda for unif} with is represented by solid arrows in the first panel. The lifted plan $\gamma_\theta^{\mu^1,\mu^2}$ defined in \eqref{eq: gamma for unif} is represented in the second panel by dashed assignments. 
    (b) Visualization for finite discrete measures $\mu^1=0.1\delta_{x_1}+0.3\delta_{x_2}+0.6\delta_{x_3}$ (green circles), $\mu^2=0.5\delta_{y_1}+0.3\delta_{y_2}+0.2\delta_{y_3}$ (blue circles). When projection according a given direction $\theta$, the locations with green masses $0.3$ and $0.6$ overlap, as well as the locations with blue masses $0.2$ and $0.3$. Thus, the mass of  $\theta_\#\mu^1$ is concentrated at two green points (triangles) on the line determined by $\theta$, each one with $0.1$ and $0.9$  of the total mass, and similarly $\theta_\#\mu^2$ is concentrated at two points (blue triangles) each one with $0.5$ of the total mass.
\end{remark}

Now, let us prove Lemma \ref{lem: gamma_theta discrete}, that is, let us show that each measure $\gamma_\theta^{\mu^1,\mu^2}$ defined in \eqref{eq: gamma theta for general discrete} is a transport plan in $\Gamma(\mu^1,\mu^2)$.

\begin{proof}[Proof of Lemma \ref{lem: gamma_theta discrete}]
    Let $x\in\mathbb{R}^d$. First, if $p(x)=0$, then $u_\theta^{\mu^1,\mu^2}(x,y)=0$ for every $y\in\mathbb{R}^d$, and so $\gamma_\theta^{\mu^1,\mu^2}(\{x\}\times\mathbb{R}^d)=0=p(x)=\mu^1(\{x\})$. Now, assume that $p(x)\not =0$, then
    \begin{align*}
        \sum_{y\in \mathbb{R}^d}u_\theta^{\mu^1,\mu^2}(x,y)&=\frac{p(x)}{P(\bar x^\theta)}\sum_{y\in\mathbb{R}^d: \, q(y)\not=0}\frac{q(y)}{Q(\bar y^\theta)}\Lambda_\theta^{\mu^1,\mu^2}(\{(\bar x^\theta,\bar y^\theta)\})\\
        &=\frac{p(x)}{P(\bar x^\theta)}\sum_{\bar y^\theta\in \mathbb{R}^d/{\sim_\theta}: \, Q(\bar y^\theta)\not=0}\left(\sum_{y\in\bar y^\theta}q(y)\right)\frac{1}{Q(\bar y^\theta)}\Lambda_\theta^{\mu^1,\mu^2}(\{(\bar x^\theta,\bar y^\theta)\})\\
        &=\frac{p(x)}{P(\bar x^\theta)}\sum_{\bar y^\theta\in \mathbb{R}^d/{\sim_\theta}: \,  Q(\bar y^\theta)\not=0} Q(\bar y^\theta)\frac{1}{Q(\bar y^\theta)}\Lambda_\theta^{\mu^1,\mu^2}(\{(\bar x^\theta,\bar y^\theta)\})\\
        &=\frac{p(x)}{P(\bar x^\theta)}\sum_{\bar y^\theta\in \mathbb{R}^d/{\sim_\theta}} \Lambda_\theta^{\mu^1,\mu^2}(\{(\bar x^\theta,\bar y^\theta)\})=\frac{p(x)}{P(\bar x^\theta)}P(\bar x^\theta)=p(x).
    \end{align*}
Thus, $\gamma_\theta^{\mu^1,\mu^2}(\{x\}\times\mathbb{R}^d)=p(x)=\mu^1(\{x\})$ for every $x\in\mathbb{R}^d$.
 Similarly, $\sum_{x\in\mathbb{R}^d}u_\theta^{\mu^1,\mu^2}(x,y)=q(y)$, or equivalently, 
$\gamma_\theta^{\mu^1,\mu^2}(\mathbb{R}^d\times \{y\})=q(x)=\mu^2(\{y\})$ for every $y\in\mathbb{R}^d$ .   
\end{proof}

% \begin{remark}
%     For each fixed $z\in \mathbb{R}^d$, the function  $\theta \mapsto \theta \cdot z$ is continuous.
% \end{remark}

\begin{remark}[Expected Sliced Transport for uniform discrete measures]
    Let $\mu^1,\mu^2\in\mathcal{P}_{(N)}(\mathbb{R}^d)$ of the form 
$\mu^1=\frac{1}{N}\sum_{i=1}^N\delta_{x_i}$, $\mu^2=\frac{1}{N}\sum_{j=1}^N\delta_{y_j}$. Then,
the expected sliced transport plan between $\mu^1$ and $\mu^2$, $\bar\gamma^{\mu^1,\mu^2}=\mathbb{E}_{\theta\sim \sigma}[\gamma_\theta^{\mu^1,\mu^2}]$,  defines  
 a discrete measure on $\mathcal{P}(\mathbb R^d\times \mathbb R^d)$  supported on $\{(x_i,y_j)\}_{i,j\in[N]}$ where it takes the values
\begin{equation}\label{eq: bar gamma for unif}
  \bar\gamma^{\mu^1,\mu^2}(\{x_i,y_j\})=\int_{\mathbb S^{d-1}}\gamma_\theta^{\mu^1,\mu^2}(\{x_i,y_j\})d\sigma(\theta) \qquad \forall i\in [N], j\in [N].  
\end{equation}
Thus, it can be regarded as an $N\times N$ matrix whose $(i,j)$-entry is given by \eqref{eq: bar gamma for unif}. Moreover,  each $N\times N$ matrix $u_\theta^{\mu^1,\mu^2}$ defined by \eqref{eq: gamma theta for uniform} can be obtained by swapping rows from the 
 $N\times N$ identity matrix multiplied by $1/N$, there are finitely many matrices (precisely, $N!$ matrices in total). Hence, the function $\theta\mapsto u_\theta^{\mu^1,\mu^2}$ is a piece-wise constant matrix-valued function.  
 Thus, the function $\theta\mapsto\gamma_\theta^{\mu^1,\mu^2}$ (where $\gamma_\theta^{\mu^1,\mu^2}$ is an in \eqref{eq: gamma for unif}) 
 is a measurable function. This can be generalized for any pair of finite discrete measures as in the following remarks. %(since its range is discrete and finite). %\ak{not very clear argument!}

% In this case the Expected Sliced Transport discrepancy between $\mu^1$ and $\mu^2$  can be written as
% \begin{equation}
%     \mathcal{D}(\mu^1,\mu^2):= \sum_{i=1}^n\sum_{j=1}^n c(x_i,y_j)\bar{\gamma}^{\mu^1,\mu^2}(\{(x_i,y_j)\}).
%     \label{eq:est}
% \end{equation}
\end{remark}

\begin{remark}[Expected Sliced Transport for finite discrete measures]\label{remark: finite case arbitrary}
Consider arbitrary \textit{finite discrete} measures $\mu^1=\sum_{i=1}^n p(x_i)\delta_{x_i}$    and $\mu^2=\sum_{j=1}^m q(y_i)\delta_{y_j}$, i.e., discrete measures with finite support. 
\begin{itemize}
    \item Fix $x_i,y_j\in\mathbb{R}^d$, then $\theta\mapsto \frac{p(x_i)q(y_j)}{P(\bar x_i^\theta)Q(\bar y_j^\theta)}\not=1$ for all but a finite number of directions. This is due to the fact that only for finitely many directions $\theta$ we obtain overlaps of the projected points $\theta\cdot x$. 
    \item The optimal transport plan $\Lambda_\theta^{\mu^1,\mu^2}$ is given by ``matching from left to right until fulfilling  the target bins'': that is, one has to order the points similarly as in \eqref{eq: order} and consider an  ``increasing'' assignment plan. Since the order of $\{\theta\cdot x_i\}_{i=1}^n$ and $\{\theta\cdot y_j\}_{j=1}^m$
changes a finite number of times when varying $\theta\in \mathbb S^{d-1}$, the function
    $\theta \mapsto \Lambda_\theta^{\mu^1,\mu^2}$ takes a finite number of possible transportation plan options.
\end{itemize}
Thus,  
the range of  $\theta\mapsto  u_\theta^{\mu^1,\mu^2}$ is finite.   
\end{remark}

\begin{remark}[$\bar{\gamma}^{\mu^1,\mu^2}$ is well-defined for finite discrete measures]
First, we notice that for each $\theta$, the support of $\Lambda_\theta^{\mu^1,\mu^2}$ is finite or countable, and so the support of $\gamma_\theta^{\mu^1,\mu^2}$ is also finite or countable.
Given an arbitrary point $(x,y)\in \mathbb{R}^d\times \mathbb{R}^d$, we have to justify that the function $\mathbb{S}^{d-1}\ni\theta\mapsto u_\theta^{\mu^1,\mu^2}(x,y)$ is (Borel)-measurable:
%, or, equivalently, that $\mathbb{S}^{d-1}\ni\theta\longmapsto \frac{\Lambda_\theta^{\mu^1,\mu^2}(\bar x^\theta,\bar y^\theta)}{P(\bar x^\theta)Q(\bar y^\theta)}$ is (Borel)-measurable:  
If the supports of $\mu^1$ and $\mu^2$ are finite, by Remark \ref{remark: finite case arbitrary}, $\theta\mapsto  u_\theta^{\mu^1,\mu^2}$ is a piece-wise constant function, and so it is measurable and integrable on the sphere.
For the general case, when the supports of $\mu^1$ and $\mu^2$ are countable, we refer to Remark \ref{rem: well defined countable support}. 
%TO DO: MAYBE HERE WE WILL NEED TO ASK COMPACT (COUNTABLE) SUPPORTS ... 
\end{remark}

\begin{lemma}[$\bar{\gamma}^{\mu^1,\mu^2}$ is a transportation plan between $\mu^1$ and $\mu^2$]
\label{lem:barGammaMarginals}
We have that $\bar{\gamma}^{\mu^1,\mu^2}\in \Gamma(\mu^1,\mu^2)$, i.e., it has marginals $\mu^1$ and $\mu^2$. This is because for each $\theta \in \mathbb S^{d-1}$, $\gamma_\theta^{\mu^1,\mu^2}\in \Gamma(\mu^1,\mu^2)$ and because $\sigma$ is a probability measure on $\mathbb{S}^{d-1}$. 
Then,  $\bar{\gamma}^{\mu^1,\mu^2}$ is a convex combination of transport plans  $\gamma_\theta^{\mu^1,\mu^2}$, and since $\Gamma(\mu^1,\mu^2)$ is a convex set, we obtain that$\bar{\gamma}^{\mu^1,\mu^2}\in \Gamma(\mu^1,\mu^2)$.
Precisely, for every test function $\phi:\mathbb R^d\to\mathbb{R}$
\begin{align*}
    \int_{\mathbb{R}^d\times \mathbb{R}^d} \phi(x) d\bar\gamma^{\mu^1,\mu^2}(x,y)&=    \int_{\mathbb{S}^{d-1}}\int_{\mathbb{R}^d\times \mathbb{R}^d} \phi(x) d\gamma_\theta^{\mu^1,\mu^2}(x,y)d\sigma(\theta)\\
    &=\int_{\mathbb S^{d-1}}\int_{\mathbb{R}^d} \phi(x) d\mu^1(x)d\sigma(\theta)=\int_{\mathbb{R}^d} \phi(x) d\mu^1(x)
\end{align*}
Similarly,
$\int_{\mathbb{R}^d\times \mathbb{R}^d} \psi(y) d\bar\gamma^{\mu^1,\mu^2}(x,y)=\int_{\mathbb{R}^d} \psi(y) d\mu^2(y),$
and so $\bar\gamma^{\mu^1,\mu^2}$ has marginals $\mu^1$ and $\mu^2$. 
\end{lemma}

\subsection{An auxiliary result}

For simplicity, in this paper we consider the strictly convex cost $\|x-y\|^p$ ($1< p<\infty$). Also, in this section we consider $\sigma=\mathcal{U}(\mathbb{S}^{d-1})$ and in this case we denote $d\sigma(\theta)=d\theta$.

\begin{proposition}\label{prop: convergence of plan}
    Let $\Omega\subset\mathbb{R}^d$ be a compact set, and let  $\mu^1,\mu^2\in\mathcal{P}(\Omega)$. Let $(\mu_n^1)_{n\in\mathbb{N}}, (\mu_n^2)_{n\in\mathbb{N}}\subset\mathcal{P}(\Omega)$ be sequences such that, for $i=1,2$, $\mu_n^i \rightharpoonup^*\mu^i$ as $n\to\infty$,  where the limit is in the weak*-topology. For each $n\in \mathbb N$, consider optimal transportation plans $\gamma_n\in \Gamma^*(\mu^1_n,\mu^2_n)$. %\ak{They are not necessarely unique in $\mathbb{R}^d$!! So we may want to say ``for a choice of plans" }. 
    Then, there exists a subsequence such that $\gamma_{n_k}\rightharpoonup^*\gamma$, for some optimal transportation plan $\gamma\in\Gamma^*(\mu^1,\mu^2)$.  
\end{proposition}

\begin{proof}
    As $(\gamma_n)_{n\in\mathbb{N}}$ is a sequence of probability measures, their mass is 1, by Banach-Alaoglu Theorem, there exists a subsequence such that $\gamma_{n_k}\rightharpoonup^*\gamma$, for some $\gamma\in\mathcal{P}(\Omega\times \Omega)$. It is easy to see that the limit $\gamma$ has marginals $\mu^1,\mu^2$. Indeed, given any test functions $\phi,\psi\in C(\Omega)$, since each $\gamma_n$ has marginals $\mu_n^1,\mu_n^2$, we have
    \begin{align*}
        \int_{\Omega\times \Omega} \phi(x) d\gamma_n(x,y)=\int_{\Omega} \phi(x) d\mu_n^1(x) \quad \text{ and } \quad         \int_{\Omega\times \Omega} \psi(y) d\gamma_n(x,y)=\int_{ \Omega} \psi(y) d\mu_n^2(y)
    \end{align*}
    and taking limit as $n\to\infty$, we obtain 
        \begin{align*}
        \int_{\Omega\times \Omega} \phi(x) d\gamma(x,y)=\int_{\Omega} \phi(x) d\mu^1(x) \quad \text{ and } \quad         \int_{\Omega\times \Omega} \psi(y) d\gamma(x,y)=\int_{ \Omega} \psi(y) d\mu^2(y).
    \end{align*}
    Now, we only have to prove the optimality of $\gamma$ for the OT problem between $\mu^1$ and $\mu^2$. Since $(x,y)\mapsto\|x-y\|^p$ is continuous and $\gamma_n$, $\gamma$ are compactly supported, by using that for each $n\in\mathbb N$, $\gamma_n$ is optimal for the OT problem between $\mu^1_n$ and $\mu^2_n$, we have
    \begin{align}
        \lim_{n\to\infty} \left(W_p(\mu^1_n,\mu^2_n)\right)^p&=\lim_{n\to\infty}\int_{\Omega\times \Omega}\|x-y\|^p d\gamma_n(x,y)\notag  \qquad \\
        &=\int_{\Omega\times \Omega}\|x-y\|^p d\gamma(x,y)\geq \left(W_p(\mu^1,\mu^2)\right)^p \label{aux: >=}
    \end{align}    
Also, by hypothesis and \cite[Theorem 5.10]{santambrogio2015optimal} we have that  for any $\nu\in\mathcal{P}(\Omega)$,     
    \begin{equation*}
    \lim_{n\to\infty}W_p(\mu_n^1,\nu)=W_p(\mu^1,\nu) \qquad \text{ and } \qquad \lim_{n\to\infty}W_p(\nu,\mu_n^2)=W_p(\nu,\mu^2).
    \end{equation*}
    So, by using the the triangle inequality for the $p$-Wasserstein distance we get
    \begin{align}
        \lim_{n\to\infty}W_p(\mu^1_n,\mu_n^2)&\leq \lim_{n\to\infty}W_p(\mu^1_n,\mu^1)+\lim_{n\to\infty} W_p(\mu^2,\mu_n^2)+W_p(\mu^1,\mu^2)\notag\\
        &=0+W_p(\mu^1,\mu^2)=W_p(\mu^1,\mu^2)\label{aux: <=}
    \end{align}
    Therefore, from \eqref{aux: <=} and \eqref{aux: >=} we have that
    \begin{equation*}
        \lim_{n\to\infty} W_p(\mu_n^1,\mu^2_n)=W_p(\mu^1,\mu^2).
    \end{equation*}
    In particular, in \eqref{aux: >=} we have that the following equality holds:
    $$\int_{\Omega\times \Omega}\|x-y\|^p d\gamma(x,y)= \left(W_p(\mu^1,\mu^2)\right)^p.$$
    As a result, $\gamma$ is optimal for the OT problem between  $\mu^1$ and $\mu^2$.
\end{proof}

\subsection{Finite discrete measures with rational weights}

Let us denote by $\mathcal{P}_{\mathbb Q}(\mathbb{ R}^d)$ the set of finite discrete probability measures in $\mathbb{R}^d$ with rational weights, that is, $\mu\in \mathcal{P}_{\mathbb Q}(\mathbb{R}^d)$ if and only if it is of the form $\mu=\sum_{i=1}^m q_i\delta_{x_i}$ with $x_i\in\mathbb{R}^d$, $q_i \in\mathbb Q$  $\forall i\in [m]$ for some $m\in \mathbb N$, and $\sum_{i=1}^m q_i=1$.
We have
$$\mathcal{P}_{(N)}(\mathbb{R}^d)\subset \mathcal{P}_{\mathbb Q}(\mathbb{R}^d), \qquad \forall N\in\mathbb N.$$
In the definition of an uniform discrete measure $\mu=\frac{1}{N}\sum_{i=1}^N\delta_{x_i}\in\mathcal{P}_{(N)}(\mathbb{R}^d)$ one can allow $x_i=x_j$ for some pairs of indexes $i\not =j$. 

\begin{proposition}\label{prop: rational metric}
    $\mathcal{D}_p(\cdot,\cdot)$ defined by \eqref{eq:est 2} is a metric in $\mathcal{P}_{\mathbb Q}(\mathbb{ R}^d)$.
\end{proposition}
\begin{proof}
    This was essentially pointed out Remark \ref{remark: PWD}: When restricting to the space $\mathcal{P}_{(N)}(\mathbb{R}^d)$, our $\mathcal{D}_p(\cdot,\cdot)$ and the Projected Wasserstein distance presented in \citep{rowland2019orthogonal} coincide. \cite{rowland2019orthogonal} prove the metric property.  We recall here their main argument, which is used for showing the triangle inequality.   
    Given $\mu^1,\mu^2,\mu^3\in\mathcal{P}_{(N)}(\mathbb{R}^d)$ of the form $\mu^1=\frac{1}{N}\sum_{i=1}^N\delta_{x_i}$, $\mu^2=\frac{1}{N}\sum_{i=1}^N\delta_{y_i}$, $\mu^3=\frac{1}{N}\sum_{i=1}^N\delta_{z_i}$. Fix $\theta\in\mathbb{S}^{d-1}$, and consider permutations $\zeta_\theta,\tau_\theta, \xi_\theta\in\mathbf S_N$, so that 
\begin{align*}
    &\theta\cdot x_{\zeta_\theta^{-1}(1)}\leq \dots \leq \theta\cdot x_{\zeta_\theta^{-1}(N)},\\
    &\theta\cdot y_{\tau_\theta^{-1}(1)}\leq \dots \leq \theta\cdot y_{\tau_\theta^{-1}(N)},\\
    &\theta\cdot z_{\xi_\theta^{-1}(1)}\leq \dots \leq \theta\cdot z_{\xi_\theta^{-1}(N)}
\end{align*}
Thus, the key idea is that 
\begin{align*}
    \mathcal{D}_p(\mu^1,\mu^2)^p&=\int_{\mathbb S^{d-1}}\frac{1}{N}\sum_{i=1}^N\|x_{\zeta_\theta^{-1}(i)}-y_{\tau_\theta^{-1}(i)}\|^pd\theta=\sum_{\zeta,\tau,\xi\in\mathbf{S}_N}\frac{\mathbf{q}(\zeta,\tau,\xi)}{N}\sum_{i=1}^N\|x_{\zeta^{-1}(i)}-y_{\tau^{-1}(i)}\|^p\\
    \mathcal{D}_p(\mu^2,\mu^3)^p&=\int_{\mathbb S^{d-1}}\frac{1}{N}\sum_{i=1}^N\|y_{\tau_\theta^{-1}(i)}-z_{\xi_\theta^{-1}(i)}\|^pd\theta=\sum_{\zeta,\tau,\xi\in\mathbf{S}_N}\frac{\mathbf{q}(\zeta,\tau,\xi)}{N}\sum_{i=1}^N\|y_{\tau^{-1}(i)}-z_{\xi^{-1}(i)}\|^p\\
    \mathcal{D}_p(\mu^3,\mu^1)^p&=\int_{\mathbb S^{d-1}}\frac{1}{N}\sum_{i=1}^N\|z_{\xi_\theta^{-1}(i)}-x_{\zeta_\theta^{-1}(i)}\|^pd\theta=\sum_{\zeta,\tau,\xi\in\mathbf{S}_N}\frac{\mathbf{q}(\zeta,\tau,\xi)}{N}\sum_{i=1}^N\|z_{\xi^{-1}(i)}-x_{\zeta^{-1}(i)}\|^p
\end{align*}
where $\mathbf{q}\in \mathcal{P}(\mathbf{S}_N\times \mathbf{S}_N\times \mathbf{S}_N)$ is such that $\mathbf{q}(\zeta,\tau,\xi)$ is the probability that the tuple permutations $(\zeta,\tau,\xi)=(\zeta_\theta,\tau_\theta,\xi_\theta)$ are required, given that $\theta$ is drawn from $\Unif(\mathbb S^{d-1})$.
With these alternative expressions established by the authors in \citep{rowland2019orthogonal}, the triangle inequality follows from the standard Minkowski inequality for weighted finite $L^p$-spaces.  

%This proof can be generalized when using any $\sigma\in\mathcal{P}(\mathbb{S}^{d-1})$ with $supp(\sigma)=\mathbb{S}^{d-1}$ in place of $Unif(\mathbb{S}^{d-1})$.

Finally, notice that we have used the fact that each $\mu^{1}$ is associated to $N$-indexes $\{1,\dots,N\}$, without asking that the points $\{x_i\}$ do not overlap, i.e., they could be repeated. That is, given $\mu^1=\frac{1}{N}\sum_{i=1}^n\delta_{x_i}\in\mathcal{P}_{(N)}(\mathbb{R}^d)$ one can allow $x_i=x_j$ for some pairs of indexes $i\not =j$
(analogously for $\mu^2$ and $\mu^3$). Thus, the proof also holds for measures in $\mathcal{P}_\mathbb{Q}(\mathbb{R}^d)$: Indeed, let $\mu^1,\mu^2,\mu^3\in\mathcal{P}_{\mathbb Q}(\mathbb{R}^d)$ be of the form $\mu^1=\sum_{i=1}^{n_1}\frac{r_i^1}{s_i^1}\delta_{x_i}$, $\mu^2=\sum_{i=1}^{n_2}\frac{r_i^2}{s_i^2}\delta_{y_i}$, $\mu^3=\sum_{i=1}^{n_3}\frac{r_i^3}{s_i^3}\delta_{z_i}$ with $r_i^j,q_i^j\in\mathbb{N}$. First, consider the $n_1'$ as the
least common multiple of $\{s^1_1,\dots,s_{n_1}^1\}$ and, for each $i\in [n_1]$, let $\tilde r_i^1$ so that $\frac{\tilde r_i^1}{n_1'}=\frac{r_i^1}{s_i^1}$. Thus, we can rewrite $\mu^1=\sum_{i=1}^{n_1}\frac{\tilde r_i^1}{n_1'}\delta_{x_i}$. Notice that, since $\mu^1$ is a probability measure, we have $n_1'=\sum_{i=1}^{n_1}{\tilde r_i^1}$. Now,
for each $i\in [n_1]$ such that $\tilde r_i^1>1$,  consider $\tilde r_i^1$ copies of the corresponding point $x_i$ so that we can rewrite $\mu^1=\sum_{i=1}^{n_1'}\frac{1}{n_1'}\delta_{x_i}$ (where 
 we recall that $n_1'=\sum_{i=1}^{n_1}{\tilde r_i^1}$ and the points $x_i$ in the new expression can be repeated, i.e., they are not necessarily all different).  Repeat this process to rewrite $\mu^2=\sum_{i=1}^{n_2'}\frac{1}{n_2'}\delta_{y_i}$, $\mu^3=\sum_{i=1}^{n_3'}\frac{1}{n_3'}\delta_{z_i}$. Now, consider $N$ as the least common multiple of $n_1',n_2',n_3'$, and rewrite the measures as $\mu^1=\frac{1}{N}\sum_{i=1}^{N}\delta_{x_i}$, $\mu^2=\frac{1}{N}\sum_{i=1}^{N}\delta_{y_i}$, $\mu^3=\frac{1}{N}\sum_{i=1}^{N}\delta_{z_i}$ where the points $x_i$, $y_i$, $z_i$ can be repeated if needed. Thus,  $\mu^1,\mu^2,\mu^3$ can be regarded as measures in $\mathcal{P}_{(N)}(\mathbb{R}^d)$ where $\mathcal{D}_p(\cdot,\cdot)$ behaves as a metric. 
\end{proof}

\subsection{The proof for general finite discrete measures}

We first introduce some notation. Consider a finite discrete probability measure $\mu\in\mathcal{P}(\mathbb{R}^d)$ of the form $\mu=\sum_{i=1}^{m}p^i\delta_{x_i}$, with general weights $p^i\in \mathbb{R}_+$ such that $\sum_{i=1}^m p^i=1$. For each $i\in \{1,\dots, m-1\}$, %\mt{[Should this be to $m-1$?]} 
consider an increasing sequence of rational numbers $(p_n^i)_{n\in\mathbb N}\subset \mathbb Q$, with $0\leq p_n^i\leq p^i$, such that $\lim_{n\to\infty}p^i_n=p^i$. For $i=m$, consider the sequence $(p_n^m)_{n\in\mathbb N}\subset \mathbb Q$ defined by $0\leq p_n^m:=1-\sum_{i=1}^{m-1}p_n^i\leq 1$. Thus, $\lim_{n\to\infty}p_n^m=1-\lim_{n\to\infty}\sum_{i=1}^{m-1}p_n^i=1-\sum_{i=1}^{m-1}p^i=p^m$.  Define the sequence of probability measures $(\mu_n)_{n\in\mathbb N}$ given by $\mu_n:=\sum_{i=1}^m p^i_n\delta_{x_i}\in \mathcal{P}_\mathbb{Q}(\mathbb{R}^d)$.
It is easy to show that $(\mu_n)_{n\in\mathbb N}$ converges to $\mu$ in Total Variation (i.e., uniform convergence or strong convergence): Indeed, let $\varepsilon>0$. For each $i\in [m]$, let $N_i\in\mathbb N$ such that $|p^i_n-p^i|<\varepsilon/m$ $\forall n\geq N_i$ and define $N=\max\{N_1,\dots, N_m\}$. Now, given any set $B\subset \mathbb{R}^d$ we obtain, for $n\geq N$,
\begin{align*}
    |\mu_n(B)-\mu(B)|&=\left|\sum_{i\in[m]: \, x_i\in B} (p_n^i-p^i)\right| \qquad \text{($\mu_n$ and $\mu$ have the same support)}\\
    &\leq \sum_{i\in[m]: \, x_i\in B}|p_n^i-p^i| \qquad \text{(triangle inequality)} \\
    &\leq \sum_{i\in [m]}|p_n^i-p^i| \qquad \text{(sum over all indexes to get independence of the set $B$)}\\
    &<\varepsilon. 
\end{align*}
This shows that $\lim_{n\to \infty}\|\mu_n-\mu\|_{\TV}=0$. Moreover, this shows that in this case, i.e., when approximating a finite discrete measure $\mu$ by a sequence of measures having the same support as $\mu$, we only care about point-wise convergence. 

We will now introduce some lemmas which together with the above proposition will allow us to prove the metric property of $\mathcal{D}_p(\cdot,\cdot)$ for finite discrete probability measures. For all of them we will consider:
\begin{itemize}
    \item $\mu^1,\mu^2$ two finite discrete probability measures in $\mathbb{R}^d$ given by $\mu^1=\sum_{i=1}^{m_1}p^i\delta_{x_i}$, $\mu^2=\sum_{j=1}^{m_2}q^j\delta_{y_j}$
    \item $(\mu^1_n)_{n\in\mathbb{N}}$, $(\mu^2_n)_{n\in\mathbb{N}}$ approximating sequences of probability measures $\mu^1_n=\sum_{i=1}^{m_1}p^i_n\delta_{x_i}$, $\mu^2_n=\sum_{j=1}^{m_2}q^j_n\delta_{y_j}$, with rational weights $\{p^i_n\}$, $\{q^j_n\}$, defined in analogy to what we have already done, i.e., so that, for $k=1,2$ we have that $(\mu^k_n)_{n\in \mathbb N}$ is a sequence of probability measures that converges to $\mu^k$ in Total Variation).
\end{itemize}   
Also, for each $\theta$,  $\Lambda_\theta^{\mu^1,\mu^2}$ denotes the unique optimal transport plan between $\theta_\#\mu^1$ and $\theta_\#\mu^2$; $\gamma_\theta^{\mu^1,\mu^2}$ denotes lifted transport plan between $\mu^1$ and $\mu^2$ given as in \eqref{eq: gamma theta for general discrete}; and $\bar \gamma^{\mu^1,\mu^2}$ the expected sliced transport plan between $\mu^1$ and $\mu^2$ given as in \eqref{eq: bar gamma 2}. Similarly, for each $n\in \mathbb{N}$ we consider the plans $\Lambda_\theta^{\mu^1_n,\mu^2_n}$, $\gamma_\theta^{{\mu^1_n,\mu^2_n}}$, and  $\bar \gamma^{{\mu^1_n,\mu^2_n}}$.

\begin{lemma}\label{lemma: conv Lambda discrete}
    The sequence  $\left(\Lambda_\theta^{\mu^1_n,\mu_n^2}\right)_{n\in \mathbb N}$ converges to $\Lambda_\theta^{\mu^1,\mu^2}$ in Total Variation.
\end{lemma}

\begin{lemma}\label{lemma: conv gamma theta discrete}
    The  sequence  $\left(\gamma_\theta^{\mu^1_n,\mu_n^2}\right)_{n\in  \mathbb N}$ converges to $\gamma_\theta^{\mu^1,\mu^2}$ in Total Variation.
\end{lemma}

\begin{lemma}\label{lemma: conv bar gamma discrete}
    The sequence  $\left(\bar\gamma^{\mu_n^1,\mu_n^2}\right)_{n\in \mathbb{N}}$  converges to $\bar\gamma^{\mu^1,\mu^2}$ in Total Variation. 
\end{lemma}

\begin{lemma}\label{lemma: conv D discrete}
$\displaystyle \lim_{n\to\infty}\mathcal{D}_p(\mu^1_{n},\mu^2_{n})=\mathcal{D}_p(\mu^1,\mu^2)$. 

    % Consider the convergent subsequence 
    % $\left(\bar\gamma^{\mu_{n_k}^1,\mu_{n_k}^2}\right)_{k\in \mathbb N}$ from Lemma \ref{lemma: conv bar gamma discrete} and the corresponding subsequences $\left(\mu_{n_k}^1\right)_{k\in \mathbb N}$, $\left(\mu_{n_k}^2\right)_{k\in \mathbb N}$. Then, 
    % $\lim_{k\to\infty}\mathcal{D}_p(\mu^1_{n_k},\mu^2_{n_k})=\mathcal{D}_p(\mu^1,\mu^2)$. %\ak{$\mu^1_{n_k},\mu^2_{n_k}$ have not been defined. Of course they are subsequences but we should make this explicit.}
\end{lemma}
%\ak{We do not say in what sense the convergence in the previous lemmas are. They are all in TV an we should mention that!!}
In general, notice that since
for every $i\in [m_1]$, $j\in[m_2]$ we have that
$\lim_{n\to\infty }p_n^i=p^i$, and $\lim_{n\to\infty }q_n^j=p^j$, then  we obtain that $\lim_{n\to\infty }P_n^i=P^i$, and $\lim_{n\to\infty }Q_n^j=Q^j$, where $P^i=\sum_{i\in [m_1]: x_i\in \bar x_i^\theta} p^i$, $Q^j=\sum_{j\in [m_2]: y_j\in \bar y_j^\theta} q^j$, and where, for each $n\in \mathbb N$, $P_n^i$ and $Q_n^j$ are analogously defined (see Subsection \ref{subsec: discrete measures}). Thus,
\begin{equation}\label{eq: aux weights discrete general case}
    \lim_{n\to\infty}\frac{p^i_n q^j_n}{P^i_nQ^j_n}=\frac{p^iq^j}{P^iQ^j} \qquad \forall  i\in [m_1], j\in [m_2].
\end{equation}

\begin{proof}[Proof of Lemma \ref{lemma: conv Lambda discrete}]
    The support of all the measures we are considering are finite and so, the measures have compact support. Hence, we can apply Proposition \ref{prop: convergence of plan} to $\theta_\#\mu^i$, $(\theta_\#\mu^i)_{n\in \mathbb{N}}$, $i=1,2$.  Specifically, given $(\Lambda_{\theta}^{\mu^1_n,\mu^2_n})_{n\in\mathbb N}\in \Gamma^*(\theta_\#\mu_n^1,\theta_\#\mu_n^2)$, there exists a subsequence $(\Lambda_{\theta}^{\mu^1_{n_k},\mu^2_{n_k}})_{K\in\mathbb N}$ and $\Lambda_\theta\in\Gamma^*(\theta_\#\mu^1,\theta_\#\mu^2)$ such that
    \begin{equation}\label{eq: *conv lambda}
        \Lambda_{\theta}^{\mu^1_{n_k},\mu^2_{n_k}}\rightharpoonup^* \Lambda_\theta.
    \end{equation}
    As we are in one dimension, the set $\Gamma^*(\theta_\#\mu^1,\theta_\#\mu^2)$ is a singleton, and so we have that $\Lambda_\theta=\Lambda_\theta^{\mu^1,\mu^2}$ is the unique optimal transport plan. 
    Since the supports of all the measures are the same, (that is, $\{(\theta\cdot x_i,\theta\cdot y_j)\}_{i\in[m_1],j\in[m_2]}$), the weak$^*$ convergence in \eqref{eq: *conv lambda} implies the stronger convergence in Total Variation.  

    Now, suppose that the original sequence $(\Lambda_{\theta}^{\mu^1_n,\mu^2_n})_{n\in\mathbb N}$ does not converge to $\Lambda_\theta^{\mu^1,\mu^2}$ (in Total Variation). Then, given $\varepsilon>0$, there exists a subsequence $(\Lambda_{\theta}^{\mu^1_{n_j},\mu^2_{n_j}})_{j\in\mathbb N}$ such that 
    \begin{equation}\label{eq: not conv}
      \|\Lambda_{\theta}^{\mu^1_{n_j},\mu^2_{n_j}}-\Lambda_\theta^{\mu^1,\mu^2}\|_{\TV}>\varepsilon  
    \end{equation}
    But again, from Proposition \ref{prop: convergence of plan}, using that the supports of all the measures involved are the same set $\{(\theta\cdot x_i,\theta\cdot y_j)\}_{i\in[m_1],j\in[m_2]}$, and the fact that $\Gamma^*(\theta_\#\mu^1,\theta_\#\mu^2)=\{\Lambda_\theta^{\mu^1,\mu^2}\}$ (only one optimal transport plan), we have that there exists a sub-subsequence such that
    \begin{equation*}
        \|\Lambda_{\theta}^{\mu^1_{n_{j_i}},\mu^2_{n_{j_i}}}-\Lambda_\theta^{\mu^1,\mu^2}\|_{\TV}<\varepsilon.
    \end{equation*}
    contradicting \eqref{eq: not conv}. Since the contradiction is achieved from assuming that  the whole sequence $(\Lambda_{\theta}^{\mu^1_n,\mu^2_n})_{n\in\mathbb N}$ does not converge to $\Lambda_\theta^{\mu^1,\mu^2}$, we have that, in fact, it does converge to $\Lambda_\theta^{\mu^1,\mu^2}$ in Total Variation. 
\end{proof}

\begin{proof}[Proof of Lemma \ref{lemma: conv gamma theta discrete}]
    This holds by looking at \eqref{eq: gamma theta for general discrete}: Due to the fact that the supports of $\mu_n^1$ and $\mu^1$ are the same (respectively, for $\mu_n^2$ and $\mu^2$), we only care about the locations $\{(x_i,y_j)\}_{i\in[m_1],j\in[m_2]}$, and then  by using \eqref{eq: aux weights discrete general case} and the convergence from Lemma \ref{lemma: conv Lambda discrete}, the result holds true.
\end{proof}

\begin{proof}[Proof of Lemma \ref{lemma: conv bar gamma discrete}]
    As pointed out before, we only care about point-wise convergence: That is, since the supports of the measures involved coincide (are the same set $\{(x_i,y_j)\}_{i\in[m_1],j\in[m_2]}$) weak$^*$ convergence, point-wise convergence and convergence in Total Variation are equivalent. 
    
    Since $0\leq \gamma_\theta^{\mu^1_n,\mu^2_n}(\{(x_i,y_j)\})\leq 1$ and $\mathbb{S}^{d-1}$ is compact, by the convergence result from Lemma \ref{lemma: conv gamma theta discrete} and using the Dominated Convergence Theorem, we have that for each $i\in[m_1], j\in [m_2]$,
    \begin{align*}
       \lim_{n\to \infty} \bar \gamma^{\mu_{n}^1,\mu^2_{n}}(\{(x_i,y_j)\})&=\lim_{n\to \infty}\int_{\mathbb S^{d-1}}\gamma_\theta^{\mu_{n}^1,\mu^2_{n}}(\{(x_i,y_j)\})d\theta\\
       &=\int_{\mathbb S^{d-1}}\lim_{n\to \infty}\gamma_\theta^{\mu_{n}^1,\mu^2_{n}}(\{(x_i,y_j)\})d\theta\\
       &=\int_{\mathbb S^{d-1}}\gamma_\theta^{\mu^1,\mu^2}(\{(x_i,y_j)\})d\theta=\bar \gamma^{\mu^1,\mu^2}(\{(x_i,y_j)\}) \qedhere
    \end{align*} 
\end{proof}

\begin{proof}[Proof of Lemma \ref{lemma: conv D discrete}]
    \begin{align}\label{eq: aux: d lim}
        \left|\mathcal{D}_p(\mu_{n}^1,\mu_{n}^2)^p-\mathcal{D}_p(\mu^1,\mu^2)^p\right|\leq \max_{i\in [m_1], j\in [m_2]}\{\|x_i-y_j\|^p\}\|\bar \gamma^{\mu^1_{n},\mu^2_{n}}-\bar \gamma^{\mu^1,\mu^2}\|_{\TV}
    \end{align}
    where the RHS goes to $0$ as $n\to \infty$, due to Lemma \ref{lemma: conv bar gamma discrete}. 
\end{proof}

\begin{theorem}
    $\mathcal{D}_p(\cdot,\cdot)$ is a metric for the space of finite discrete probability measures in $\mathbb{R}^d$.
\end{theorem}

\begin{proof}\, 
\begin{itemize}
    \item Symmetry: The way we constructed $\mathcal{D}_p(\cdot,\cdot)$ makes it so that  $\mathcal{D}_p(\mu^1,\mu^2)=\mathcal{D}_p(\mu^2,\mu^1)$.
    \item Positivity: It is clear that by definition $\mathcal{D}_p(\mu^1,\mu^2)\geq 0$.
    \item Identity of indiscernibles: 

    %$\theta_{\#}\mu^1 = \theta_{\#}\mu^2$ then $\gamma_\theta^{\mu^1,\mu^2} = (id\times id)_{\#}\mu$ for all $\theta\in\bbS^{d-1}$. Hence $\bar{\gamma}^{\mu^1,\mu^2} = (id\times id)_{\#}\mu$ which implies $\mathcal{D}_p(\mu^1,\mu^2) = 0$.
    
    First, if $\mu^1=\mu^2=:\mu$, then  
    $\gamma_\theta^{\mu^1,\mu^2} = (id\times id)_{\#}\mu$ for all $\theta\in\bbS^{d-1}$. Hence $\bar{\gamma}^{\mu^1,\mu^2} = (id\times id)_{\#}\mu$ which implies $\mathcal{D}_p(\mu^1,\mu^2) = 0$.
    
    %$\theta_{\#}\mu^1 = \theta_{\#}\mu^2$ for all $\theta\in \mathbb S^{d-1}$ 
    %and so $\Lambda_\theta^{\mu^1,\mu^2}$ is induced by the one dimensional identity map  
    %on $supp(\theta_\#\mu)$, and so the corresponding lifted plan $\gamma_\theta^{\mu^1,\mu^2}$ is induced by the one dimensional identity map  
    %on $supp(\mu)=\{x_i\}_{i=1}^n$. Thus, the average plan $\bar{\gamma}^{\mu^1,\mu^2}$ is also induced by the one dimensional identity map  
    %on $supp(\mu)$ (this is due to the fact that $supp(\sigma)=\mathbb{S}^{d-1}$, for e.g., if $\sigma = Unif(\mathbb{S}^{d-1})$). This implies that  $$\mathcal{D}_p(\mu^1,\mu^2)^p=\sum_{i,j}\|x_i-x_j\|^p\bar\gamma^{\mu^1,\mu^2}(\{x_i,x_j\})=\sum_{i}\|x_i-x_i\|^p\mu(\{x_i\})=0.$$

    Secondly, if $\mu^1,\mu^2$ are such that $\mathcal{D}_p(\mu^1,\mu^2)=0$, by having 
    $$W_p(\mu^1,\mu^2)\leq \mathcal{D}_p(\mu^1,\mu^2)=0,$$
    we can use the fact that $W_p(\cdot,\cdot)$ satisfies the identity of indiscernibles by being a distance. That is,  
    $W_p(\mu^1,\mu^2)=0$ implies $\mu^1=\mu^2$.
    \item Triangle inequality: Given $\mu^1,\mu^2,\mu^3$ arbitrary finite discrete measures with arbitrary real weights, consider approximating sequences $(\mu^1_n)_{n\in\mathbb{N}}$, $(\mu^2_n)_{n\in\mathbb{N}}$, $(\mu^3_n)_{n\in\mathbb{N}}$ in $\mathcal{P}_\mathbb{Q}(\mathbb R^d)$ as before. Notice that every subsequence of 
    $(\mu^1_n)_{n\in\mathbb{N}}$ (respectively of $(\mu^2_n)_{n\in\mathbb{N}}$ and $(\mu^3_n)_{n\in\mathbb{N}}$) will converge to $\mu^1$ (respectively, to $\mu^2$ and $\mu^3$), as every subsequence of a convergent sequence is convergent. 
    
    By Proposition \ref{prop: rational metric}, we have that, for each $n\in \mathbb R^d $,
    $$\mathcal{D}_p(\mu^1_n,\mu^2_n)\leq\mathcal{D}_p(\mu^1_n,\mu^3_n)+\mathcal{D}_p(\mu^3_n,\mu^2_n).$$
    Taking the limit as $n\to\infty$, from Lemma \eqref{lemma: conv D discrete} we obtain 
\[ \mathcal{D}_p(\mu^1,\mu^2)\leq\mathcal{D}_p(\mu^1,\mu^3)+\mathcal{D}_p(\mu^3,\mu^2). \qedhere \]
\end{itemize}
\end{proof}

\subsection{Discrete measures with countable support}

\begin{lemma}\label{lem:no_overlap}
Let $\mu=\sum_{m\in\mathbb N}^\infty p^m\delta_{x_m}$ be a discrete probability measure with countable support $\{x_m\}_{m\in\mathbb N}$.
%\mt{[I think we also need to define $\nu$.]}
Let $\sigma$ be an absolutely continuous probability measure with respect to the Lebesgue measure on the sphere (we write, $\sigma\ll\Unif(\mathbb{S}^{d-1})$). 
Let $$S_\mu:=\{\theta\in \mathbb{S}^{d-1}: \,  \theta\cdot x_m=\theta \cdot x_{m'} \text{ for some pair }(m,m') \text{ with } m\not=m'\}.$$
Then $$\sigma(S_\mu)=0.$$
\end{lemma}
\begin{proof}
%It suffices to prove $\sigma^2_{\mathbb{S}^{d-1}}(S_\mu)=0$.
First, consider distinct points $x_m,x_{m'}$ on the support of $\mu$, and let $$S(x_m,x_{m'})=\{\theta\in \mathbb{S}^{d-1}: \, \theta \cdot x_m=\theta \cdot x_{m'}\}.$$
It is straightforward to verify that
$$S(x_m,x_{m'})=\mathbb{S}^{d-1}\cap \text{span}(\{x_m-x_{m'}\})^\perp,$$
where $\text{span}(\{x_m-x_{m'}\})^\perp$ is the orthogonal subspace to the line in the direction of the vector $(x_m-x_{m'})$. Thus, $S(x_m,x_{m'})$ is a subset of a $d-2$-dimensional sub-sphere in $\mathbb{S}^{d-1}$, and therefore 
\begin{align}
  \sigma_{\mathbb{S}^{d-1}}(S(x_{m},x_{m'}) )=0\label{pf:simga_ S_ij}  
\end{align}
Since, 
\begin{align}
S_\mu = \bigcup_{(x_m,x_{m'})\in \textbf{M}} S(x_m,x_{m'}), \qquad \text{ where } \quad \textbf{M}=\{(x_m,x_{m'}):\, m\neq m'\} 
\end{align}
we have that 
\begin{align}
\sigma(S_\mu)\leq \sum_{(x_i,x_i')\in \textbf{M}} \sigma(S(x_m,x_{m'}))=0 \nonumber. 
\end{align}
since $\textbf{M}$ is countable
(indeed, $|\textbf{M}|\leq |{\supp}(\mu)\times {\supp}(\mu)|$).
% In addition, we have 
% \begin{align}
% S_\mu = \bigcup_{(x_i,x_i')\in A} S(x_i,x_{i'})
% \end{align}
% and thus  
% \begin{align}
% \sigma_{\mathbb{S}^{d-1}}(S_\mu)\leq \sum_{(x_i,x_i')\in A} \sigma_{\mathbb{S}^{d-1}}(S(x_i,x_{i'}))=\sum_{(x_i,x_{i'})\in A} 0=0 \nonumber. 
% \end{align}
%Similarly, $\sigma_{\mathbb{S}^{d-1}}(S_\nu)=0$ and thus $\sigma_{\mathbb{S}^{d-1}}(S_\mu\cup S_\nu)=0$ and we complete the proof. 
\end{proof}

\begin{remark}\label{rem: well defined countable support}[$\bar{\gamma}^{\mu^1,\mu^2}$ is well-defined for discrete measures with countable support] Let $\sigma \ll \Unif(\mathbb S^{d-1})$.
Given two discrete probability measures $\mu^1=\sum p(x)\delta_x$ and $\mu^2=\sum q(y)\delta_y$ with countable support, from Lemma \ref{lem:no_overlap}, we have that $\sigma(S_{\mu^i})=0$ for $i=1,2$, and so 
$\sigma(S_{\mu^1}\cup S_{\mu^2})=0$.
Therefore, similarly to the case of discrete measures with finite support, given any  $x\in \supp(\mu^1)$, $y\in \supp(\mu^2)$ we have that the map $\theta\mapsto \frac{p(x)q(y)}{P(\bar x^\theta)Q(\bar y^\theta)}$ from $\mathbb S^{d-1}$ to $\mathbb R$ is equal to the constant function $\theta\mapsto 1$ up to a set of $\sigma$-measure $0$. This implies that the function $\theta\mapsto u_\theta^{\mu^1,\mu^2}$ is measurable. Finally, since $|\gamma_\theta^{\mu^1,\mu^1}(\{(x,y)\})|\leq 1$ for every $(x,y)$, we have that $\bar{\gamma}^{\mu^1,\mu^2}$ is well-defined. %Moreover, from the proof of Lemma \ref{lem:no_overlap}, 
%the function
%     $\theta \mapsto \Lambda_\theta^{\mu^1,\mu^2}$ takes a countable number of possible transportation plan options ???
% Thus,  
% the range of  $\theta\mapsto  u_\theta^{\mu^1,\mu^2}$ is ...   
% , and so it is integrable on the sphere (with respect to the given measure $\sigma$). THIS NEEDS TO BE COMPLETED. \ak{Yes, or just not do it at all!!}
\end{remark}

\section{Equivalence with \texorpdfstring{Weak$^*$}{Weak*} Convergence}\label{app: weak*}

%\mt{[Apologies, I'm probably using different notation (particularly at the end of the proof I may be inconsistent) --- please can someone check for consistency and correct it?]}

\begin{lemma}\label{lem: weak con for projection}
    Let $\Omega\subset \mathbb{R}$ be a compact set, $\mu\in \mathcal{P}(\Omega)$ and consider a sequence of probability measures $(\mu_n)_{n\in \mathbb N}$ defined in $\Omega$ such that $\mu_n\weakstarto\mu$ as $n\to\infty$. Then, for each $\theta\in\mathbb S^{d-1}$, we have that $\theta_\#\mu_n\weakstarto\theta_\#\mu$ as $n\to\infty$. 
\end{lemma}
\begin{proof} 
Given $\theta\in\mathbb S^{d-1}$, notice that $\{\theta\cdot x: \, x\in \Omega\}$ is a $1$-dimensional compact set, which contains the supports of $\theta_\#\mu$ and $(\theta_\#\mu_n)_{n\in\mathbb N}$. Thus, when dealing with the weak$^*$-topology we can use continuous functions as test functions. 
Let $\varphi:\mathbb R\to \mathbb R$ be a continuous test function, then 
    \begin{align*}
        \int_{\mathbb R} \varphi(u)d\theta_\#\mu_n(u)=\int_{\mathbb R^d}\varphi(\theta\cdot x)d\mu_n(x)\underset{n\to\infty}{\longrightarrow} \int_{\mathbb R^d}\varphi(\theta\cdot x)d\mu(x)= \int_{\mathbb R} \varphi(u)d\theta_\#\mu(u)
    \end{align*}
    since the composition $x\mapsto \theta \cdot x\mapsto \varphi(\theta \cdot x)$ is a continuous function from $\mathbb R^d$ to $\mathbb R$.
\end{proof}

\begin{lemma}\label{lem: lambda theta weak conv}
    Let $\Omega\subset \mathbb{R}$ be a compact set, $\mu^i\in \mathcal{P}(\Omega)$, $i=1,2$, and consider  sequences of probability measures $(\mu_n^i)_{n\in \mathbb N}$ defined in $\Omega$ such that $\mu_n^i\weakstarto\mu^i $ as $n\to\infty$, for $i=1,2$. Given $\theta\in\mathbb S^{d-1}$, consider $\Lambda_\theta^{\mu^1,\mu^2}$ the unique optimal transport plan between  $\theta_\#\mu^1$ and $\theta_\#\mu^2$, and for each $n\in\mathbb{N}$, consider $\Lambda_\theta^{\mu_n^1,\mu_n^2}$ the unique optimal transport plan between  $\theta_\#\mu^1_n$ and $\theta_\#\mu^2_n$. Then 
    $\Lambda_\theta^{\mu_n^1,\mu_n^2}\weakstarto\Lambda_\theta^{\mu^1,\mu^2}$.
\end{lemma}
\begin{proof}
    The proof is similar to that of Lemma \ref{lemma: conv Lambda discrete}. From Lemma \ref{lem: weak con for projection}, Proposition \ref{prop: convergence of plan}, and uniqueness of optimal plans in one-dimension, there exists a subsequence $(\Lambda_{\theta}^{\mu^1_{n_k},\mu^2_{n_k}})_{k\in\mathbb N}$ such that
    \begin{equation*}
        \Lambda_{\theta}^{\mu^1_{n_k},\mu^2_{n_k}}\rightharpoonup^* \Lambda_\theta^{\mu^1,\mu^2}.
    \end{equation*}
    Now, suppose that the original sequence $(\Lambda_{\theta}^{\mu^1_n,\mu^2_n})_{n\in\mathbb N}$ does not converge to $\Lambda_\theta^{\mu^1,\mu^2}$ (in the weak$^*$-topology). Thus, there exists a continuous function $\varphi:\mathbb R\to\mathbb R$ such that for a given $\varepsilon>0$ there exists a subsequence $(\Lambda_{\theta}^{\mu^1_{n_j},\mu^2_{n_j}})_{j\in\mathbb N}$ with
    \begin{equation}\label{eq: not conv2}
      \left|\int_{\mathbb R} \varphi(u) \, d\Lambda_{\theta}^{\mu^1_{n_j},\mu^2_{n_j}}(u)-\int_{\mathbb R} \varphi(u) \, d\Lambda_\theta^{\mu^1,\mu^2}(u)\right|>\varepsilon  
    \end{equation}
    But again, from Lemma \ref{lem: weak con for projection}, Proposition \ref{prop: convergence of plan}, and uniqueness of optimal plans in one-dimension, we have that there exists a sub-subsequence such that
    \begin{equation*}
     \int_{\mathbb R} \varphi(u) \, d\Lambda_{\theta}^{\mu^1_{n_{j_i}},\mu^2_{n_{j_i}}}(u)\underset{i\to\infty}{\longrightarrow}  \int_{\mathbb R} \varphi(u) \, d\Lambda_\theta^{\mu^1,\mu^2}(u),
    \end{equation*}
    contradicting \eqref{eq: not conv2}. Since the contradiction is achieved from assuming that  the whole sequence $(\Lambda_{\theta}^{\mu^1_n,\mu^2_n})_{n\in\mathbb N}$ does not converge to $\Lambda_\theta^{\mu^1,\mu^2}$, we have that, in fact, it does converge to $\Lambda_\theta^{\mu^1,\mu^2}$ in the weak$^*$-topology.  
\end{proof}

\begin{lemma}\label{lem:B(gamma)_converge}
Consider probability measures supported in a compact set $\Omega\subset\mathbb R^d$ such that $\mu^1_n\weakstarto\mu^1, \mu^2_n\weakstarto \mu^2$. Let $\theta\in \mathbb{S}^{d-1}$, then 
% and for each $n$, let $\Lambda_{\theta}^{\mu^1_n,\mu^2_n}$ denote the optimal transportation plan for $W_p((\theta)_\#\mu^1_n,(\theta)_\#\mu^2_n)$. Then there exists a sub-sequence such that 
% $\Lambda_{\theta}^{\mu^1_{n_k},\mu^2_{n_k}}\weakstarto \Lambda_\theta^{\mu^1,\mu^2}$, where $\Lambda_\theta^{\mu^1,\mu^2}$ is optimal for $W_p((\theta)_\#\mu^1,(\theta)_\#\mu^2)$. 
the sequence of lifted plans $\{\gamma_{\theta}^{\mu^1_n,\mu^2_n}\}_{n\in\mathbb N}$ satisfies that there exists a subsequence such that
$\gamma_{\theta}^{\mu^1_{n_k},\mu^2_{n_k}}\weakstarto \gamma^*$ for $$\gamma^*\in\Gamma(\mu^1,\mu^2;\Lambda_\theta^{\mu^1,\mu^2}):=\{\gamma\in\Gamma(\mu^1,\mu^2): \, (\theta\times\theta)_\#\gamma=\Lambda_\theta^{\mu^1,\mu^2}\},$$
where $\theta\times\theta(x,y):=(\theta\cdot x,\theta \cdot y)$ for all $(x,y)\in\mathbb R^d\times \mathbb{R}^d$.
\end{lemma}

\begin{proof}
% By Proposition \ref{prop: convergence of plan}, and the fact $(\theta)_\#\mu^1_n\weakstarto (\theta)_\#\mu^1, (\theta)_\#\mu^2_n\weakstarto (\theta)_\#\mu^2$, 
% there exists a $\gamma^*_\theta$ as a minimizer of $W_p^p((\theta)_\#\mu^1,(\theta)_\#\mu^2)$ and the weak-convergence limit of a sub-sequence $\gamma_{\theta}^{n_k}$. 

Since $\gamma_\theta^{\mu^1_n,\mu^2_n}\in\Gamma(\mu^1_n,\mu^2_n)$, similar to the proof of Proposition \ref{prop: convergence of plan}, by Banach-Alaoglu Theorem, there exists a subsequence $\gamma_\theta^{\mu^1_{n_k},\mu^2_{n_k}}$, such that $\gamma_\theta^{\mu^1_{n_k},\mu^2_{n_k}}\weakstarto \gamma^*,$
for some $\gamma^*\in\mathcal{P}(\Omega\times \Omega)$.
Again, as in the proof of Proposition \ref{prop: convergence of plan}, it can be shown that 
$\gamma^*\in\Gamma(\mu^1,\mu^2)$. 

In addition, we have 
$$(\theta\times\theta)_\#\gamma_\theta^{\mu^1_{n_k},\mu^2_{n_k}}\weakstarto(\theta\times\theta)_\#\gamma^*\qquad \text{and} \qquad (\theta\times\theta)_\#\gamma_\theta^{\mu^1_{n_k},\mu^2_{n_k}}=\Lambda_\theta^{\mu^1_{n_k},\mu^2_{n_k}}\weakstarto \Lambda_\theta^{\mu^1,\mu^2}$$
(where the first one follows by similar arguments to those in Lemma \ref{lem: weak con for projection} and the second one follows by definition of the lifted plans).
% \begin{align}
% \begin{cases}
% &  \\
% &    
% \end{cases}
% \end{align}
By the uniqueness of the limit (for the weak$^*$-convergence), we have $(\theta\times\theta)_\#\gamma^*=\Lambda_\theta^{\mu^1,\mu^2}$. Thus, $\gamma^*\in\Gamma(\mu^1,\mu^2;\Lambda_\theta^{\mu^1,\mu^2})$. 
\end{proof}
\begin{theorem}
\label{prop:B_SlicedDtoD}
Let $\sigma\ll Unif(\mathbb S^{d-1})$. Consider discrete probability measures measures $\mu^1=\sum_{i=1}^\infty p_i\delta_{x_i},\mu^2=\sum_{j=1}^\infty q_j\delta_{y_j}$ in $\Omega$ supported on a compact set $\Omega\subset\mathbb R^d$.
Consider sequences $(\mu^1_n)_{n\in\mathbb N}$, $(\mu^2_n)_{n\in\mathbb N}$ of discrete probability measures defined on $\Omega$ such that $\mu^1_n\rightharpoonup^* \mu^1, \mu^2_n\rightharpoonup^* \mu^2$. Then  $\sigma$-a.s. we have that $\mathcal{D}_p(\mu^1_n,\mu^2_n;\theta)\to \mathcal{D}_p(\mu^1,\mu^2;\theta)$ as $n\to \infty$.
%where 
%$$\mathcal{D}_p(\mu^1,\mu^2;\theta)^p:=\sum_{x,y}\|x-y\|^p\, \gamma_\theta^{\mu^1,\mu^2}(\{(x,y)\})$$
Moreover,  $\mathcal{D}_p(\mu^1_n,\mu^2_n)\to \mathcal{D}_p(\mu^1,\mu^2)$ as $n\to\infty$. 
\end{theorem}
\begin{proof}
% For each pair of discrete measures $(\mu^1_n,\mu^2_n)$, we let $\Lambda_\theta^{\mu^1_n,\mu^2_n}$ denote the optimal transportation plan for $W_p(\theta_\#\mu^1_n,\theta_\#\mu^2_n)$. 
%Similarly, we let  $\gamma_\theta^n$ to denote one optimal transportation plan for $W_p^p(\theta_\#\mu^{1,n},\theta_\#\mu^{2,n})$. 
Let us define the set
$$S(\mu^1,\mu^2):=\left\{\theta\in\mathbb{R}^d: \,  \Gamma(\mu^1,\mu^2;\Lambda_\theta^{\mu^1,\mu^2})= \{\gamma_\theta^{\mu^1,\mu^2}\}\right\}.$$

Since we are considering discrete measures, notice that $$\mathbb{S}^{d-1}\setminus S(\mu^1,\mu^2)\subseteq S_{\mu^1}\cup S_{\mu^2},$$
where $S_{\mu^1}=\{\theta\in \mathbb{S}^{d-1}: \, \theta\cdot x_m=\theta\cdot x_{m'} \text{ for some pair }m\neq m' \}$ and $S_{\mu^2}=\{\theta\in \mathbb{S}^{d-1}: \, \theta\cdot y_n=\theta\cdot y_{n'} \text{ for some pair }n\neq n' \}$.  

By Lemma \ref{lem:no_overlap}, we have $\sigma(\mathbb{S}^{d-1}\setminus S(\mu^1,\mu^2))\leq \sigma(S_{\mu^1}\cup S_{\mu^2})=0$.
Thus, 
$$\sigma(S(\mu^1,\mu^2))=1.$$
% for
% $$S:=\bigcap_{n=1}^\infty S(\mu^1_n,\mu^2_n)\cap S(\mu^1,\mu^2).$$
% we have $\sigma(S)=1$. 
% Let $\theta\in S$, %by Proposition \ref{prop: convergence of plan}, there exists a sub-sequence $\gamma_\theta^{n_k}$ that weakly converges to $\gamma_\theta$ where $\gamma_\theta $ is optimal for $W_p^p(\theta_\#\mu^1,\theta_\#\mu^2)$. 
%And 

Let $\theta\in S(\mu^1,\mu^2)$, and consider the lifted plans $\gamma_\theta^{\mu^1,\mu^2}$ and $\gamma_\theta^{\mu^1_{n},\mu^2_{n}}$. By Lemma \ref{lem:B(gamma)_converge}, there exists a subsequence of $(\gamma_\theta^{\mu^1_{n},\mu^2_{n}})_{n\in\mathbb{N}}$ such that  $$\gamma_\theta^{\mu^1_{n_k},\mu^2_{n_k}}\weakstarto \gamma^*\in\Gamma(\mu^1,\mu^2;\Lambda_\theta^{\mu^1,\mu^2}).$$
Since $\theta\in S(\mu^1,\mu^2)$, we have that $\Gamma(\mu^1,\mu^2;\Lambda_\theta^{\mu^1,\mu^2})$ contains only one element, which is $\gamma_\theta^{\mu^1,\mu^2}$. Hence, $$\gamma^*=\gamma_\theta^{\mu^1,\mu^2}.$$

Moreover, by uniqueness of weak convergence,  proceeding similarly as in the proof of Lemma \ref{lem: lambda theta weak conv}, we have that the whole sequence $(\gamma_\theta^{\mu^1_{n},\mu^2_{n}})_{n\in\mathbb N}$ converges to $\gamma_\theta^{\mu^1,\mu^2}$ in the weak$^*$-topology. 

Therefore, by definition of weak$^*$-convergence for measures supported in a compact set (in our case, $\Omega\times \Omega\subset \mathbb R^d\times \mathbb R^d$), since $(x,y)\mapsto\|x-y\|^p$ is a continuous function, we have 
\begin{align}
\lim_{n\to\infty}\mathcal D_p(\mu^1_{n},\mu^2_{n};\theta)^p&=\lim_{n\to\infty} \int_{\Omega^2} \|x-y\|^p d\gamma^{\mu^1_{n},\mu^2_{n}}_\theta (x,y) \nonumber\\
&= \int_{\Omega^2} \|x-y\|^p d\gamma^{\mu^{1},\mu^{2}}_\theta(x,y)\nonumber\\
&=\mathcal{D}_p(\mu^{1},\mu^2;\theta)^p\nonumber
\end{align}
%(notice that $\gamma_\theta^{\mu^1,\mu^2}$ and $\gamma_\theta^{\mu^1_{n_k},\mu^2_{n_k}}$ are in fact discrete measures, so we could replace the integrals $\int$ by $\sum$ ).

% Note, for each sub-sequence sequence of $D(\mu^{1,n_k},\mu^{2,n_k};\theta)$, we can find a sub-sub-sequence which converges to $D(\mu^1,\mu^2;\theta)$, thus we have 
% $$\lim_{n\to\infty }D(\mu^1_n,\mu^2_n;\theta)\to D(\mu^1,\mu^2;\theta),\forall \theta\in S.$$

Combining this with the fact that $\sigma(S(\mu^1,\mu^2))=1$ and that $(\mathcal D_p(\mu^1_{n},\mu^2_{n};\theta)^p)_{n\in\mathbb N}$ is bounded, that is, $|\mathcal D_p(\mu^1_{n},\mu^2_{n};\theta)^p|\leq \max_{(x,y)\in \Omega\times \Omega}\|x-y\|^p$, by Dominated Lebesgue Theorem we obtain 
\begin{align*}
    \lim_{n\to\infty}\mathcal{D}_p(\mu^1_{n},\mu^2_{n})^p&=\lim_{n\to\infty}\int_{\mathbb S^{d-1}}\mathcal{D}_p(\mu^1_{n},\mu^2_{n};\theta)^pd\sigma(\theta)\notag\\
    &=\int_{\mathbb S^{d-1}}\lim_{n\to\infty}\mathcal{D}_p(\mu^1_{n},\mu^2_{n};\theta)^pd\sigma(\theta)\notag\\
    &=\int_{\mathbb S^{d-1}}\mathcal{D}_p(\mu^1,\mu^2;\theta)^pd\sigma(\theta)\notag\\
    &=\mathcal{D}_p(\mu^1,\mu^2)^p
\end{align*}

\end{proof}

\begin{corollary}
Let $\mu,\mu_n\in \cP(\Omega)$, where $\Omega\subset\bbR^d$ is compact, be of the form $\mu = \sum_{x\in\bbR^d} p(x) \delta_{x}$, $\mu_n = \sum_{x\in\bbR^d} p_n(x) \delta_{x}$ where $p(x)$ and $p_n(x)$ are 0 at all but countably many $x\in\bbR^d$.
Assume $\sigma\ll Unif(\mathbb S^{d-1})$. 
Then, $\cD_p(\mu_n,\mu)\to 0$ if and only if $\mu_n\weakstarto \mu$.
\end{corollary}

\begin{proof}
If $\cD_p(\mu_n,\mu)\underset{n\to\infty}{\longrightarrow} 0$ then, by Remark \ref{remark: W<D}, $W_p(\mu_n,\mu)\underset{n\to\infty}{\longrightarrow} 0$, hence $\mu_n\underset{n\to\infty}{\weakstarto}\mu$.

The converse is a Corollary of Theorem \ref{prop:B_SlicedDtoD}
\end{proof}

\end{document}

%% file: math_commands.tex
\usepackage{amsmath,amsfonts,bm}

\def\eqref#1{equation~\ref{#1}}
\def\1{\bm{1}}

\DeclareMathAlphabet{\mathsfit}{\encodingdefault}{\sfdefault}{m}{sl}
\SetMathAlphabet{\mathsfit}{bold}{\encodingdefault}{\sfdefault}{bx}{n}